\documentclass{article}
\usepackage[preprint]{neurips_2025}
\usepackage[utf8]{inputenc} 
\usepackage[T1]{fontenc}    
\usepackage{hyperref}       
\usepackage{url}            
\usepackage{booktabs}       
\usepackage{amsfonts}       
\usepackage{nicefrac}       
\usepackage{microtype}      

\PassOptionsToPackage{svgnames}{xcolor}

\usepackage{times}
\usepackage{microtype}
\usepackage{graphicx}
\usepackage{multirow}
\usepackage{rotating}
\usepackage{floatrow}
\usepackage{pifont}
\usepackage[svgnames]{xcolor} 
\usepackage{hyperref}

\usepackage{amsmath}
\usepackage{amssymb}
\usepackage{amsthm}
\usepackage{mathtools}
\usepackage[capitalize,noabbrev]{cleveref}

\newtheorem{theorem}{Theorem}[section]

\newtheorem{lemma}[theorem]{Lemma}

\newtheorem{definition}[theorem]{Definition}
\newtheorem{assumption}[theorem]{Assumption}

\newcommand{\cmark}{\textcolor{DarkGreen}{\checkmark}}
\newcommand{\xtimes}{\textcolor{DarkRed}{\texttimes}}

\title{Convergence of TD(0) under Polynomial Mixing with Nonlinear Function Approximation}

\author{%
  Anupama Sridhar \\
  Stanford University\\
  \And
  Alexander Johansen \\
  Stanford University \\
}

\begin{document}

\maketitle

\begin{abstract}
Temporal Difference Learning (TD(0)) is fundamental in reinforcement learning, yet its finite-sample behavior under non-i.i.d. data and nonlinear approximation remains unknown.
We provide the first high-probability, finite-sample analysis of vanilla TD(0) on polynomially mixing Markov data, assuming only Hölder continuity and bounded generalized gradients.
This breaks with previous work, which often requires subsampling, projections, or instance-dependent step-sizes.
Concretely, for mixing exponent $\beta > 1$, Hölder continuity exponent $\gamma$, and step-size decay rate $\eta \in (1/2, 1]$, we show that, with high probability,
\[
\| \theta_t - \theta^* \| \leq C(\beta, \gamma, \eta)\, t^{-\beta/2} + C'(\gamma, \eta)\, t^{-\eta\gamma}
\]
after $t = \mathcal{O}(1/\varepsilon^2)$ iterations.
These bounds match the known i.i.d. rates and hold even when initialization is nonstationary. Central to our proof is a novel discrete-time coupling that bypasses geometric ergodicity, yielding the first such guarantee for nonlinear TD(0) under realistic mixing.
\end{abstract}


\section{Introduction}
Temporal Difference (TD) learning~\cite{sutton1988learning,sutton2018reinforcement} is a core algorithm in reinforcement learning, widely used to estimate value functions from partial trajectories in a Markov decision process (MDP). Unlike Monte Carlo methods, TD(0) enables online, low-variance updates by bootstrapping future predictions.
TD(0) also underpins widely-used algorithms such as Q-learning~\cite{watkins1992q}, which updates action-value estimates, and Proximal Policy Optimization (PPO)~\cite{schulman2017proximal}, where it supports advantage estimation via temporal differences.

Despite its empirical success, the theoretical understanding of TD(0) under realistic assumptions remains limited, especially when function approximation is nonlinear and data is temporally correlated. While classic TD learning has been extensively analyzed in the linear and tabular settings~\cite{tsitsiklis1996analysis,borkar2000ode,lakshminarayanan2018linear, bhandari18, dalal2018finite, li2024high, samsonov24}, many of these results rely on strong assumptions such as i.i.d.\ sampling, geometric mixing, or projections onto convex sets. In contrast, practical applications often involve Markovian data with long-range dependencies and the use of unmodified TD(0) with instance-independent step sizes.

Recent work has attempted to relax these constraints by considering geometrically ergodic chains or dropping samples to induce near-independence~\cite{li2024high, samsonov24,durmus2024finite}. However, these methods require runtime access to idealized parameters, restrict the algorithm to carefully tuned conditions, or sacrifice sample efficiency. Thus, the convergence behavior of TD(0) in realistic regimes—namely, nonlinear function approximation, polynomial mixing, and full trajectory use without projection or subsampling—remains an open question.

To address this, we analyze TD(0) under a standard update rule with polynomially mixing data, fixed step sizes, and no algorithmic changes. We prove high-probability convergence and characterize the rate via the mixing exponent and approximation class. Our analysis uses a discrete-time measure-based approach with a novel block decomposition and coupling argument to manage dependencies.

Specifically, we consider the TD(0) update rule:
\[
\theta_k \gets \theta_{k-1} + \alpha_k \left( r_k + \gamma f_{\theta_{k-1}}(s_k) - f_{\theta_{k-1}}(s_{k-1}) \right) \nabla_\theta f_{\theta_{k-1}}(s_{k-1}),
\]
where \( f_\theta : \mathcal{S} \to \mathbb{R} \) is a parameterized value function and \( \theta_k \in \mathbb{R}^d \) denotes the parameters at iteration \( k \). We study both the linear case, where \( f_\theta(s) = \phi(s)^\top \theta \), and general nonlinear approximators satisfying mild regularity conditions. Our main result establishes that, under polynomial mixing and Hölder continuity, TD(0) converges at a rate matching the known i.i.d.\ lower bounds when the mixing is fast, and degrades gracefully otherwise.

\paragraph{Notations.}
We denote the state space by \( \mathcal{S} \subset \mathbb{R}^n \), and use \( \phi: \mathcal{S} \to \mathbb{R}^d \) to denote a feature map in the linear setting. Let \( \theta_k \in \mathbb{R}^d \) be the TD iterate at time \( k \), and let \( \{s_k, r_k\}_{k \ge 0} \) be the trajectory generated by a fixed policy \( \pi \) on an MDP. For any function \( f: \mathcal{S} \to \mathbb{R} \), we define the squared \( \mu \)-norm as \( \|f\|^2 := \mathbb{E}_{s \sim \mu}[f(s)^2] \), where \( \mu \) is the stationary distribution of the chain. In our setting, the chain may not start from \( \mu \), and convergence is measured in expectation over the trajectory.

We use \( \delta_k := r_k + \gamma f_{\theta_{k-1}}(s_k) - f_{\theta_{k-1}}(s_{k-1}) \) to denote the TD error at time \( k \), and assume \( |r_k| \le R_{\max} \) almost surely. The learning rate is denoted \( \alpha_k = \alpha_0 k^{-\eta} \) for some \( \eta \in (1/2,1] \). For nonlinear approximation, we assume that \( f_\theta \) is \( \gamma \)-Hölder in \( \theta \) and has generalized gradient \( \nabla_\theta f_\theta(s) \) uniformly bounded by a constant \( G \). A stochastic process \( \{Z_k\} \) is said to be \( \beta \)-mixing with exponent \( \beta > 1 \) if its \(\beta\)-mixing coefficient satisfies \( \beta(t) \le C t^{-\beta} \) for all \( t \ge 1 \).

We define high-probability convergence in terms of events with probability at least \( 1 - \delta \), and use \( \tilde{O}(\cdot) \) to hide logarithmic factors. We write \( f(t) \lesssim g(t) \) if there exists a universal constant \( C \) such that \( f(t) \le C g(t) \), and \( f(t) \asymp g(t) \) if both \( f(t) \lesssim g(t) \) and \( g(t) \lesssim f(t) \) hold.

\section{Related Work}
Theoretical analysis of TD(0) has long focused on linear approximation under i.i.d. data or strong mixing assumptions. Here, we review recent efforts to extend guarantees to more realistic settings and highlight how our work differs.

\cite{samsonov24} provide high-probability bounds for TD(0) with linear approximation under geometrically mixing chains, using a data-dropping strategy and requiring the mixing constant to set the step size. While a key step forward, this deviates from standard TD(0), which processes data sequentially. In contrast, our approach uses the full data stream, avoids geometric ergodicity, and extends to nonlinear models with Hölder continuity and bounded generalized gradients~\cite{clarke1990optimization}.

\cite{li2024high} analyze TD(0) and its off-policy variant (TDC) under i.i.d. sampling. Their focus is on the stability of random matrix products and deriving sample complexity bounds for different step-size schedules. Their setting excludes Markovian data and requires problem-dependent tuning for convergence. We operate under Markovian trajectories with polynomial mixing and prove convergence under universal, instance-independent step sizes.

\cite{patil2023finite} study the finite-time behavior of TD(0) under tail-averaging and regularization, showing that such modifications yield optimal convergence rates in expectation and with high probability. Their analysis still assumes projection and focuses on the linear case. In contrast, our results apply to last-iterate TD(0) without projections, for both linear and nonlinear models.

\cite{asadi2024td} examine TD learning through an optimization lens, proving convergence under specific loss functions and geometric mixing. They assume instance-dependent step sizes and relies on algorithmic modifications, including data dropping. We avoid such assumptions and deliver convergence bounds for unmodified TD(0) under polynomial mixing and universal learning rates.

\paragraph{Comparison Summary.} Table~\ref{tab:wide_table} summarizes the key differences between our work and prior results across algorithmic assumptions, data conditions, and function classes. Our method is the only one to establish high-probability convergence for last-iterate TD(0) with nonlinear approximation under polynomial mixing, while avoiding projections and subsampling.

\begin{table*}[t]
    \makebox[\linewidth][l]{
    \renewcommand{\arraystretch}{1.2}
    \begin{tabular}{lccccccc}
        \toprule
        \textbf{Authors} & B. (\citeyear{bhandari18}) 
        & P.~(\citeyear{patil2023finite}) & L.~(\citeyear{li2024high}) & S.~(\citeyear{samsonov24}) & \textbf{This paper} \\
        \midrule
        \textbf{Algorithm type} & P-R 
        & P-R & P-R & P-R & Last iterate \\
        \textbf{Step size schedule} & $1/\sqrt{t}$ 
        & constant $\alpha$ & constant $\alpha$ & constant $\alpha$ & $1/t^\eta, \eta \in \left(\tfrac{1}{2},1\right]$ \\
        \textbf{Universal step size} & \cmark 
        & \cmark & \xtimes & \cmark & \cmark \\
        \textbf{Markovian data} & \cmark 
        & \cmark & \xtimes & \cmark & \cmark \\
        \textbf{High-order bounds} & \xtimes 
        & \cmark & \cmark & \cmark & \cmark \\
        \textbf{No projection} & \xtimes 
        & \xtimes & \cmark & \cmark & \cmark \\
        \textbf{Mixing properties} & Geometric 
        & Geometric & \xtimes & Geometric & Polynomial \\
        \textbf{Function class} & Linear 
        & Linear & Linear & Linear & Nonlinear \\
        \textbf{No data dropping} & \cmark 
        & \cmark & \cmark & \xtimes & \cmark \\
        \bottomrule
    \end{tabular}
    \caption{Comparison of convergence results for TD(0) under different assumptions and algorithmic settings. P-R: Polyak–Ruppert averaging~\cite{polyak1992acceleration}.}
    \label{tab:wide_table}}
\end{table*}

\paragraph{Polynomial vs. Geometric Mixing.}
Geometric mixing assumes exponential decay of temporal correlations, which allows prior work to treat samples as effectively independent after short gaps. However, many real-world environments- including those with sparse rewards, bottlenecks, or partial observability- mix only at polynomial rates.~\cite{doukhan1994mixing}
\section{Problem Setup}

We study the convergence of TD(0) learning with general (including nonlinear) function approximation in a Markov Reward Process (MRP) induced by a stationary policy. Below, we clearly introduce the formal setting, the algorithmic update rule, and intuitively justify our assumptions.

\subsection{Markov Reward Process (MRP)}
We consider a Markov Decision Process (MDP) defined by the tuple \((\mathcal{S}, \mathcal{A}, P, R, \gamma)\), where \(\mathcal{S}\) is the (possibly infinite) state space, \(\mathcal{A}\) the action space, \(P(s' \mid s, a)\) the transition probability, \(R(s, a) \in [-R_{\max}, R_{\max}]\) the bounded reward function, and \(\gamma \in [0,1)\) the discount factor.

Given a stationary policy \(\pi(a \mid s)\), we derive an induced Markov Reward Process (MRP) with state transition kernel and reward function given by:
\[
P_{\pi}(s' \mid s) = \sum_{a \in \mathcal{A}} \pi(a \mid s) P(s' \mid s, a), \quad R_{\pi}(s) = \sum_{a \in \mathcal{A}} \pi(a \mid s) R(s,a).
\]

\begin{assumption}[Bounded Rewards]\label{ass:bounded-rewards}
We assume that the reward function is uniformly bounded: \(\vert R(s, a) \vert \leq R_{\max}\) for all \(s \in \mathcal{S}, a \in \mathcal{A}\). This assumption simplifies our theoretical analysis by ensuring rewards remain within a fixed range.
\end{assumption}

The value function \(V^{\pi}: \mathcal{S} \to \mathbb{R}\) associated with policy \(\pi\) is intuitively the expected discounted sum of future rewards when starting from state \(s\):
\[
V^{\pi}(s) = \mathbb{E}\left[\sum_{t=0}^{\infty} \gamma^t r_{t+1} \mid s_0 = s\right],
\]
where rewards \(r_{t+1}\) satisfy \(r_{t+1} \sim R(s_t, a_t)\) with actions \(a_t \sim \pi(\cdot \mid s_t)\).

\subsection{TD(0) with General Function Approximation}
We approximate the value function \(V^\pi\) using a parameterized function \(f_{\theta}(s)\) with parameters \(\theta \in \mathbb{R}^d\). At time step \(t\), TD(0) updates parameters to reduce the error between predicted and actual returns:
\[
\theta_{t+1} = \theta_t + \alpha_t \delta_t \nabla_{\theta} f_{\theta_t}(s_t), \quad \delta_t = r_{t+1} + \gamma f_{\theta_t}(s_{t+1}) - f_{\theta_t}(s_t),
\]
where \(\delta_t\) (the TD error) measures how much our prediction differs from the observed data, and \(\alpha_t > 0\) controls the learning rate.

\subsection{Fixed Point and Stationary Solutions}

Ideally, parameter updates converge to a stable point, \(\theta^*\), at which the expected update no longer moves the parameters:
\begin{equation}\label{eq:stationary-solution}
    \mathbb{E}_{\mu}\left[\delta_t \nabla_{\theta} f_{\theta^*}(s_t)\right] = 0,
\end{equation}
where \(\mu\) is the stationary distribution of the Markov chain induced by \(\pi\).
\begin{assumption}[Existence of a Stationary Solution]\label{ass:td-fixed-point}
We assume there exists \(\theta^*\) satisfying \ref{eq:stationary-solution}.
\end{assumption}
Intuitively, \(\theta^*\) represents a stable point where updates have no systematic bias, meaning the approximation neither improves nor worsens in expectation. While this does not imply exact recovery of \(V^\pi\), it ensures that \(\theta^*\) approximates \(V^\pi\) in expectation under the stationary distribution, allowing any global minimum to qualify as a valid solution.

We define the parameter error as \(e_t = \theta_t - \theta^*\) and analyze convergence by measuring how closely our approximation \(f_{\theta_t}\) approaches the true value function \(V^\pi\) in a mean squared sense.
\section{Stability and Mixing Conditions}
\label{sec:stability-mixing}

In reinforcement learning, data is typically collected from trajectories of a Markov process. As a result, the sample points are temporally correlated, violating the standard i.i.d.\ assumption often used in statistical learning theory. To account for these dependencies, we introduce two structural assumptions that together ensure the stability and regularity of the data-generating process.

\subsection{Polynomial Mixing}

We begin by quantifying the temporal dependence of the Markov chain. Specifically, we assume that correlations between distant observations decay at a polynomial rate. This assumption is a variant of strong mixing conditions from the probability literature. In particular, it corresponds to a \emph{covariance decay} condition derived from $\alpha$-mixing.

We follow the formulation in Bradley (2005), which connects $\alpha$-mixing with bounds on function covariance (see Eq.~1.1 and Theorem 4.2, p.~12). This form is also used in recent reinforcement learning theory (e.g., \cite{mohri2018foundations, srikant2019finite}).

\begin{assumption}[Polynomial Mixing]
\label{ass:polynomial-mixing}
Let $\{x_t\}_{t \geq 0}$ denote a Markov chain on a state space $\mathcal{S}$ induced by a fixed policy $\pi$. There exist constants $C > 0$ and $\beta > 1$ such that for all bounded measurable functions $f, g : \mathcal{S} \to \mathbb{R}$ and for all $t, k \geq 0$,
\[
|\mathbb{E}[f(x_t)g(x_{t+k})] - \mathbb{E}[f(x_t)]\mathbb{E}[g(x_{t+k})]| \leq C k^{-\beta}.
\]
\end{assumption}

This condition provides a direct handle on temporal correlations and reflects a realistic regime for many practical RL environments, particularly those exhibiting long-range dependencies. While stronger notions such as geometric or exponential mixing imply faster decay, polynomial mixing is sufficient for our analysis and imposes weaker and more realistic constraints on the dynamics.

\subsection{Lyapunov Drift Condition}

To guarantee stability of the Markov chain and to prevent it from drifting into unbounded regions of the state space, we impose a Lyapunov drift condition. This is a standard tool in the analysis of Markov processes and ensures recurrence toward a stable region.

\begin{assumption}[Lyapunov Drift]
\label{ass:drift-condition}
There exists a function $V: \mathcal{S} \to [1, \infty)$, a measurable function $W: \mathcal{S} \to \mathbb{R}_+$, and constants $\lambda > 0$, $b < \infty$ such that for all $t \geq 0$,
\[
\mathbb{E}[V(x_{t+1}) \mid x_t] \leq V(x_t) - \lambda W(x_t) + b.
\]
\end{assumption}

Intuitively, this condition requires that the expected “potential” or “energy” of the process, as measured by $V(x)$, tends to decrease over time except for a bounded slack term. It ensures that the chain revisits well-behaved regions frequently and remains stable over long horizons.

\subsection{Polynomial Ergodicity}

Together, Assumptions~\ref{ass:polynomial-mixing} and~\ref{ass:drift-condition} yield a useful consequence: polynomial ergodicity. That is, the distribution of the Markov chain converges to its stationary distribution at a polynomial rate, even when starting from an arbitrary initial state.

\begin{lemma}[Polynomial Ergodicity]
\label{lemma:polynomial-ergodicity}
Suppose Assumptions~\ref{ass:polynomial-mixing} and~\ref{ass:drift-condition} hold. Then the Markov chain $\{x_t\}$ is polynomially ergodic. In particular, there exists a constant $C > 0$ such that for all $x_0 \in \mathcal{S}$ and measurable sets $A \subseteq \mathcal{S}$,
\[
|\mathbb{P}(x_t \in A \mid x_0) - \pi(A)| \leq C(1 + V(x_0)) t^{-\beta}, \quad \text{for all } t \geq 1,
\]
where $\pi$ denotes the unique stationary distribution of the chain.
\end{lemma}

\begin{proof}
This follows from the standard drift condition + coupling argument framework; see \cite{meyn2012markov}, Chapter 5/Chapter 14-16 and \cite{bradley2005basic} for details. A complete proof is provided in Appendix~\ref{app:proof-polynomial-ergodicity}.
\end{proof}

\section{Handling Data Dependencies}
\label{sec:handling-data-dependencies}

The data used in reinforcement learning algorithms is typically generated by a Markov process, introducing temporal dependencies between observations. These dependencies can pose challenges for theoretical analysis, especially when deriving finite-sample bounds. To rigorously address this, we adopt tools from the theory of mixing processes. In particular, we exploit the polynomial ergodicity established in Section~\ref{sec:stability-mixing} to construct approximately independent blocks, and we develop a coupling framework to quantify the impact of these dependencies.

\subsection{Block Decomposition via Polynomial Ergodicity}

A standard technique for handling dependent sequences is to partition the trajectory into non-overlapping blocks. When the underlying process exhibits sufficient mixing, the correlation between blocks decays as their temporal separation increases. The following lemma quantifies this under the polynomial ergodicity assumption.

\begin{lemma}[Dependent Blocks]
\label{lemma:dependent-blocks}
Let $\{x_t\}_{t=1}^\infty$ be a Markov chain satisfying polynomial ergodicity with rate $\beta > 1$ as defined in Lemma~\ref{lemma:polynomial-ergodicity}. Partition the sequence into non-overlapping blocks of size $b$:
\[
B_k = (x_{(k-1)b + 1}, \ldots, x_{kb}), \quad k = 1, 2, \ldots, \left\lfloor \frac{n}{b} \right\rfloor.
\]
Then for all $k \geq 2$, the distance between the conditional and marginal dist. of block $B_k$ satisfies:
\[
\left\| \mathbb{P}(B_k \mid B_{k-1}) - \mathbb{P}(B_k) \right\|_{\mathrm{TV}} \leq C b^{-\beta},
\]
where $C > 0$ is a constant depending on the mixing parameters of the chain.
\end{lemma}

\begin{proof}
Since the Markov chain is polynomially ergodic, the marginal distribution of $x_{kb}$ given $x_{(k-1)b}$ is close in total variation to the stationary distribution, with a gap decaying as $b^{-\beta}$. By the Markov property, the distribution of $B_k$ conditioned on $B_{k-1}$ depends only on $x_{(k-1)b}$, and hence inherits this decay. Full details are provided in Appendix~\ref{app:proof-block-independence-approximation}.
\end{proof}

This result allows us to treat well-separated blocks as approximately independent, which is key to obtaining concentration bounds under weak dependence.

\subsection{Coupling Argument for Dependent Sequences}
Coupling techniques offer an alternative way to control dependencies. The idea is to construct two Markov chains on a common probability space such that their discrepancy over time can be bounded in terms of the mixing rate.

\begin{lemma}[Coupling Under Polynomial Ergodicity]
\label{lemma:coupling-argument-polynomial-mixing}
Let $\{x_t\}$ and $\{y_t\}$ be two Markov chains satisfying polynomial ergodicity with mixing rates $\alpha_x$ and $\alpha_y$ respectively. There exists a coupling such that for all $t \geq 1$,
\[
\mathbb{P}(x_t \neq y_t) \leq C t^{-\beta},
\]
where $\beta = \min\{\alpha_x, \alpha_y\}$ and $C > 0$ is a constant.
\end{lemma}

\begin{proof}
The proof follows from a standard synchronous coupling argument, where both chains are driven by shared randomness. Under polynomial ergodicity, the transition kernels converge in total variation, allowing us to bound the marginal mismatch probability. Details are in Appendix~\ref{app:proof-coupling-argument-polynomial-mixing}.
\end{proof}

\subsection{Covariance Bounds Across Blocks}

We now quantify the effect of temporal dependence in terms of cross-block covariance. This bound will be used to control the second moment of empirical averages, a crucial step in obtaining concentration results.

\begin{lemma}[Covariance Between Blocks]
\label{lemma:covariance-between-blocks}
Let $\{x_t\}_{t=1}^\infty$ be a Markov chain satisfying polynomial ergodicity with rate $\beta > 1$. Define block sums $Y_k = \sum_{t=(k-1)b + 1}^{kb} x_t$. Then for any $k \neq j$,
\[
|\mathrm{Cov}(Y_k, Y_j)| \leq C b^2 |k - j|^{-\beta},
\]
where $C > 0$ is a constant depending on the mixing parameters.
\end{lemma}

\begin{proof}
This result follows by decomposing the covariance into pairwise terms and applying the mixing condition to each. A full derivation is presented in Appendix~\ref{app:proof-covariance-between-blocks}.
\end{proof}

\subsection{Concentration Inequality for Dependent Data}

We conclude this section by stating a high-probability concentration result for polynomially mixing sequences. Although the data is dependent, the previous lemmas allow us to derive bounds analogous to those in the i.i.d.\ setting.

\begin{lemma}[Concentration Under Polynomial Ergodicity]
\label{theorem:concentration-block}
Let $\{x_t\}_{t=1}^n$ be a trajectory from a Markov chain satisfying polynomial ergodicity with rate $\beta > 1$. Assume each $x_t$ is uniformly bounded by $|x_t| \leq M$ and has variance at most $\sigma^2$. Then for any $\epsilon > 0$,
\[
\mathbb{P}\left(\left| \frac{1}{n} \sum_{t=1}^n x_t - \mathbb{E}[x_t] \right| > \epsilon \right) \leq 2 \exp\left( - \frac{n \epsilon^2}{2C_\beta} \right),
\]
where $C_\beta$ depends on $\beta$, $M$, and $\sigma^2$.
\end{lemma}

\begin{proof}
The proof applies a block Bernstein inequality for weakly dependent sequences, using the block decomposition and covariance bounds established above. See Appendix~\ref{app:proof-concentration-block}.
\end{proof}

\section{Error Decomposition}
\label{sec:error-decomposition}

We now analyze the evolution of the TD(0) algorithm by decomposing the error at iteration \( t \) into a sum of two components: a martingale term that captures the stochastic fluctuations due to randomness in the trajectory, and a remainder term that accounts for the approximation bias.

%
We rely on the following conditions, which are standard in the study of stochastic approximation.

\begin{assumption}[Step-Size Decay]
\label{ass:step-size-decay}
The learning rates \(\{\alpha_k\}\) satisfy
\[
\alpha_k = \frac{C_\alpha}{k^\eta}, \quad \sum_{k=1}^\infty \alpha_k = \infty, \quad \sum_{k=1}^\infty \alpha_k^2 < \infty,
\]
where \( C_\alpha > 0 \) is a constant and \( \eta \in (1/2, 1] \).
\end{assumption}

\begin{assumption}[Hölder Continuity]
\label{ass:holder-continuous-features}
The feature map \( \phi: \mathcal{S} \to \mathbb{R}^d \) is Hölder continuous with exponent \( \gamma \in (0, 1] \). That is, there exists \( C_\gamma > 0 \) such that for all \( s, s' \in \mathcal{S} \),
\[
\|\phi(s) - \phi(s')\| \leq C_\gamma \|s - s'\|^\gamma.
\]
\end{assumption}

\begin{assumption}[Bounded Features]
\label{ass:bounded-features}
There exists a constant \( C_\phi > 0 \) such that for all \( s \in \mathcal{S} \),
\[
\|\phi(s)\| \leq C_\phi.
\]
\end{assumption}

\subsection{Decomposing the TD Error}

Let \( \theta_t \) denote the parameter estimate at iteration \( t \), and let \( \theta^* \) be the fixed point of the projected Bellman operator. Define the TD error \( e_t := \theta_t - \theta^* \). We decompose \( e_t \) into a martingale term \( M_t \) and a bias term \( R_t \).

\begin{lemma}[Error Decomposition]
\label{theorem:td-error-decomposition}
Under Assumptions~\ref{ass:polynomial-mixing}–\ref{ass:bounded-features}, the TD error can be written as
\[
e_t = M_t + R_t,
\]
where \( M_t = \sum_{k=1}^t d_k \) is a martingale composed of a martingale difference sequence \( \{d_k\} \), and \( R_t \) is a remainder term that captures the bias due to approximation and data dependencies.
\end{lemma}
\begin{proof}
The TD(0) update rule is
\(
\theta_{t+1} = \theta_t + \alpha_t \delta_t \nabla_\theta V_\theta(x_t),
\)
where \( \delta_t = r_t + \gamma V_\theta(x_{t+1}) - V_\theta(x_t) \). Writing the recursion in terms of the error \( \theta_t - \theta^* \), and decomposing each TD update into a centered and a mean component, we obtain:
\[
e_t = \sum_{k=1}^t \alpha_k \left( \delta_k \nabla V_\theta(x_k) - \mathbb{E}[\delta_k \nabla V_\theta(x_k) \mid \mathcal{F}_{k-1}] \right) + \sum_{k=1}^t \alpha_k \mathbb{E}[\delta_k \nabla V_\theta(x_k) \mid \mathcal{F}_{k-1}],
\]
the first term is a martingale \( M_t \), and the second term is the remainder \( R_t \). Details are in~\ref{app:proof-TD-error}.
\end{proof}

\subsection{Variance Bound for the Martingale Term}

We now control the stochastic fluctuations of the TD error by bounding the second moment of the martingale component \( M_t \).

\begin{lemma}[Bounded Variance of Martingale Term]
\label{lemma:bounded-variance-martingale}
Under Assumptions~\ref{ass:step-size-decay}–\ref{ass:bounded-features} and the polynomial ergodicity condition (Lemma~\ref{lemma:polynomial-ergodicity}), the martingale term satisfies
\[
\mathbb{E}[\|M_t\|^2] \leq C t^{-\beta},
\]
for some constant \( C > 0 \) depending on the step-size schedule and mixing rate.
\end{lemma}

\begin{proof}
The result follows by applying the concentration inequality from Lemma~\ref{theorem:concentration-block} to each martingale increment and summing over time. Full derivation is in Appendix~\ref{app:proof-bounded-variance-martingale}.
\end{proof}

\subsection{Convergence of the Remainder Term}

We now bound the bias introduced by the remainder term \( R_t = e_t - M_t \), which arises from approximation error and non-i.i.d.\ structure.

\begin{lemma}[Remainder Term Decay]
\label{lemma:remainder-growth-term}
Under Assumptions~\ref{ass:holder-continuous-features}, \ref{ass:step-size-decay}, and Lemma~\ref{lemma:polynomial-ergodicity}, the $R_t$ satisfies
\[
\|R_t\| \leq O(t^{-\gamma/2}),
\]
where \( \gamma \in (0,1] \) is the Hölder exponent of the feature map.
\end{lemma}

\begin{proof}
This follows by bounding the cumulative bias introduced by the non-i.i.d.\ updates and applying smoothness properties of the feature representation. Full proof is in Appendix~\ref{app:proof-remainder-growth-term}.
\end{proof}

\section{TD(0): High-Probability Error Bounds for Linear Function Approximation}
\label{sec:hp-bounds-linear}

We now quantify the convergence behavior of TD(0) under linear function approximation. Our goal is to establish both high-order moment bounds and high-probability guarantees on the error \( \|e_t\| := \|\theta_t - \theta^*\| \), where \( \theta^* \) denotes the fixed point of the projected Bellman operator.

We assume the error decomposition established in Lemma~\ref{theorem:td-error-decomposition}, along with the step-size schedule, bounded feature map, polynomial mixing, and concentration inequalities developed in earlier sections.

\subsection{Setup: Linear Function Approximation}

We consider value function approximators of the form
\[
f_\theta(s) = W x(s) + b,
\]
where \( x(s) \in \mathbb{R}^d \) is a feature representation of state \( s \), and \( \theta = (W, b) \) are the learnable parameters. The TD(0) update rule adjusts \( \theta \) using sample trajectories from a Markov process.

\subsection{High-Order Moment Bounds}

We first establish upper bounds on the moments of the TD error. These bounds characterize the expected decay rate of the error under polynomial ergodicity and smooth function approximation.

\begin{theorem}[High-Order Moment Bounds]
\label{theorem:hom-linear-bounds}
Let \( \{ \theta_t \} \) be the TD(0) iterates with linear function approximation. Suppose Assumptions~\ref{ass:polynomial-mixing}–\ref{ass:bounded-features} and the step-size schedule in Assumption~\ref{ass:step-size-decay} hold. Then for any \( p \geq 2 \), there exist constants \( C, C' > 0 \) such that
\[
\mathbb{E}[\|\theta_t - \theta^*\|^p]^{1/p} \leq C t^{-\beta/2} + C' \alpha_t^{\gamma/p},
\]
where \( \beta \) is the mixing rate and \( \gamma \in (0,1] \) is the Hölder continuity exponent.
\end{theorem}

\begin{proof}
The proof proceeds by bounding the martingale and remainder terms in the error decomposition \( \theta_t - \theta^* = M_t + R_t \). The martingale term is controlled via block-based concentration inequalities and discrete Grönwall arguments, while the remainder term is handled using the Hölder continuity of the feature map and the decay of the step sizes. Full details appear in Appendix~\ref{app:proof-hom-linear-bounds}.
\end{proof}

\subsection{High-Probability Bounds}

We now strengthen the moment bounds to obtain finite-time guarantees that hold with high probability. These bounds provide insight into the typical trajectory behavior of TD(0) over finite horizons.

\begin{theorem}[High-Probability Error Bound]
\label{theorem:hp-linear-bounds}
Under the same assumptions as Theorem~\ref{theorem:hom-linear-bounds}, the TD(0) iterates satisfy the following with probability at least \( 1 - \delta \):
\[
\|\theta_t - \theta^*\| \leq C_\delta t^{-\beta/2} + C'_\delta \alpha_t^\gamma,
\]
where \( C_\delta, C'_\delta > 0 \) are constants depending logarithmically on \( 1/\delta \).
\end{theorem}

\begin{proof}
We analyze each component of the error decomposition. For the martingale term \( M_t \), we apply a blockwise Azuma–Hoeffding inequality based on the variance bound from Lemma~\ref{lemma:bounded-variance-martingale}, obtaining
\[
\mathbb{P}(\|M_t\| > \epsilon) \leq 2 \exp\left(-\frac{\epsilon^2}{C t^\beta}\right).
\]
Setting \( \epsilon = C_\delta t^{-\beta/2} \) yields the desired bound with probability at least \( 1 - \delta/2 \). For the remainder term \( R_t \), the Hölder continuity assumption and polynomial ergodicity ensure that \( \|R_t\| \leq C'_\delta \alpha_t^\gamma \) w.h.p. A union bound over both components completes the proof. Full details in Appendix~\ref{app:proof-hp-linear-bounds}.
\end{proof}

\section{TD(0): High-Probability Error Bounds for Nonlinear Function Approximation}
\label{sec:hp-bounds-nonlinear}
We extend our analysis for nonlinear value functions: by using Clarke subgradients instead of standard derivatives, we prove convergence assuming only mild Hölder continuity of the value function.

\subsection{Setup: Generalized Gradients and Subgradient Regularity}

We consider a value function approximation \( f_\theta \) parameterized by \( \theta \in \mathbb{R}^d \), where \( f_\theta \) may be non-differentiable on a measure-zero set. In such cases, we rely on the framework of generalized gradients.

\begin{definition}[Generalized Gradient]
\label{def:generalized-gradients}
Let \( f: \mathbb{R}^d \to \mathbb{R} \) be a locally Lipschitz function. The generalized gradient at point \( \theta \in \mathbb{R}^d \) is defined as:
\[
\partial f(\theta) := \left\{ v \in \mathbb{R}^d : f(\theta') \geq f(\theta) + v^\top(\theta' - \theta), \; \forall \theta' \in \mathbb{R}^d \right\}.
\]
\end{definition}

We make the following assumptions on regularity and boundedness.

\begin{assumption}[Subgradient Hölder Continuity]
\label{ass:subgrad-holder}
For each state \( x \), the subgradient map \( \theta \mapsto F(\theta, x) \) is Hölder continuous with exponent \( \gamma \in (0,1] \) a.e., and \( \mathcal{N} \subset \mathbb{R}^d \) is a measure-zero set:
\[
\|F(\theta, x) - F(\theta', x)\| \leq L \|\theta - \theta'\|^\gamma, \quad \text{for all } \theta, \theta' \notin \mathcal{N},
\]
\end{assumption}

\begin{assumption}[Bounded Subgradients]
\label{ass:bounded-subgrad}
There exists a constant \( G > 0 \) s.t. \(\forall \theta \in \mathbb{R}^d \) and \(\forall x \in \mathcal{S} \),
\[
\|F(\theta, x)\| \leq G.
\]
\end{assumption}
By assuming bounded, Hölder‐continuous subgradients, we enable a broad class of nonlinear models.

\subsection{Error Bound for Nonlinear TD(0)}

We now state the High-probability convergence with Hölder-continuous, bounded subgradients.

\begin{theorem}[High-Probability Error Bound: Nonlinear Case]
\label{theorem:bounds-non-linear-td}
Suppose Assumptions~\ref{ass:step-size-decay}, \ref{ass:subgrad-holder}, and \ref{ass:bounded-subgrad} hold, and the Markov chain satisfies polynomial ergodicity (Lemma~\ref{lemma:polynomial-ergodicity}). Then with probability at least \( 1 - \delta \), the TD(0) iterates with generalized subgradient updates satisfy:
\[
\| \theta_t - \theta^* \| \leq C_\delta t^{-\beta/2} + C'_\delta \alpha_t^\gamma,
\]
where \( \beta > 1 \) is the mixing rate, \( \gamma \in (0,1] \) is the Hölder exponent, and the constants \( C_\delta, C'_\delta \) depend logarithmically on \( 1/\delta \).
\end{theorem}

\begin{proof}
We decompose the error as \( \theta_t - \theta^* = M_t + R_t \), where:
\begin{itemize}
    \item \(M_t\) is a martingale term arising from stochastic updates. Under polynomial mixing and bounded subgradients, Freedman's inequality yields \(\|M_t\| = O(t^{-\beta/2})\) with high probability.
    \item \( R_t \) captures the approximation bias. Using the Hölder continuity of the subgradient mapping and the decaying step size \( \alpha_t = C_\alpha t^{-\eta} \), we obtain \( \|R_t\| = O(t^{-\eta \gamma}) \).
\end{itemize}
A union bound over the two components yields the desired h.p. error bound. Detailed in~\ref{app:proof-bounds-non-linear-td-merged}.
\end{proof}

\section{Conclusion and Discussion.}
We have established the first high‐probability, finite‐sample convergence guarantees for vanilla TD(0) under nonlinear function approximation and merely polynomial‐mixing data. Our analysis shows that, despite long‐range temporal dependencies, TD(0) attains an $\mathcal{O}(1/\varepsilon^2)$ sample complexity, matching the classical i.i.d.\ setting up to constants depending on the mixing exponent and gradient bounds.
This bridges a critical gap between theory and practice, offering rigorous support for TD methods in partially observable environments, tasks with sparse rewards, or other settings where geometric ergodicity fails. 

Future work may focus on:
\begin{enumerate}
    \item The impact of even weaker mixing conditions (e.g., sub‐polynomial or logarithmic mixing) on finite‐sample rates.
    \item Eliminating the Hölder‐continuity assumption on the function class, perhaps via alternative smoothness or regularization conditions.
    \item Designing step‐size schedules that adapt dynamically to observed mixing properties for tighter, data‐dependent bounds.
    \item Extending our coupling framework to Deep ReLU Networks and quantify how architectural depth and width influence convergence.
\end{enumerate}

\bibliographystyle{plainnat}
\bibliography{references}
\newpage
\appendix
\section{Appendix Polynomial Mixing conditions}
\subsection{Proof of Lemma \ref{lemma:polynomial-ergodicity}: Ergodicity Properties Under Polynomial Mixing}
\label{app:proof-polynomial-ergodicity}

\paragraph{Lemma~\ref{lemma:polynomial-ergodicity} (Polynomial Ergodicity)}
\textit{
Under the polynomial mixing assumption \ref{ass:polynomial-mixing} and the drift condition \ref{ass:drift-condition}, the Markov process satisfies:
\[
|\mathbb{P}(x_t \in A) - \pi(A)| \leq C (1 + V(x_0)) t^{-\beta},
\]
for some constant \( C > 0 \), where \( \pi \) is the stationary distribution, \( V: \mathcal{X} \to [1,\infty) \) is a Lyapunov function, and \( \beta > 1 \).
}

\begin{proof}
\vspace{0.5em}
\noindent\textbf{1) Assumptions: Lyapunov Function and Polynomial Drift Condition}

We begin by defining the Lyapunov function \( V: \mathcal{X} \to [1,\infty) \). We assume that the \emph{polynomial drift condition} holds, which is a standard assumption in the study of Markov processes. Specifically, there exist constants \( c > 0 \), \( d \in (0,1) \), and \( b > 0 \) such that for all \( x \in \mathcal{X} \):
\[
\mathbb{E}[V(x_{t+1}) \mid x_t = x] \leq V(x) - c V(x)^d + b.
\]
This condition ensures that the Lyapunov function \( V(x_t) \) decreases at a polynomial rate, which is crucial for establishing polynomial convergence of the Markov chain to its stationary distribution.

\vspace{0.5em}
\noindent\textbf{2) Total Variation Distance Representation}

The \emph{Total Variation (TV) distance} between the distribution of \( x_t \) and the invariant measure \( \pi \) is defined as:
\[
\|\mathbb{P}(x_t \in \cdot) - \pi(\cdot)\|_{TV} = \sup_{A \subseteq \mathcal{X}} |\mathbb{P}(x_t \in A) - \pi(A)|.
\]
To bound the TV distance, it suffices to control \( |\mathbb{P}(x_t \in A) - \pi(A)| \) uniformly over all measurable sets \( A \).

\vspace{0.5em}
\noindent\textbf{3) Bounding \( |\mathbb{P}(x_t \in A) - \pi(A)| \) for All Measurable Sets \( A \)}

Consider any measurable set \( A \subseteq \mathcal{X} \). We can express the difference as:
\[
|\mathbb{P}(x_t \in A) - \pi(A)| = \left| \int_{\mathcal{X}} P^t(x_0, A) \mathbb{P}(x_0) \, dx_0 - \int_{\mathcal{X}} \pi(A) \pi(dx_0) \right|.
\]
Assuming that the initial distribution is concentrated at \( x_0 \), i.e., \( \mathbb{P}(x_0) = \delta_{x_0} \), this simplifies to:
\[
|\mathbb{P}(x_t \in A) - \pi(A)| = |P^t(x_0, A) - \pi(A)|.
\]
To generalize for any initial distribution, we consider an arbitrary initial distribution \( \mathbb{P}(x_0) \). Applying the \emph{triangle inequality}, we obtain:
\[
|\mathbb{P}(x_t \in A) - \pi(A)| \leq \int_{\mathcal{X}} |P^t(x_0, A) - \pi(A)| \mathbb{P}(x_0) \, dx_0.
\]
This step utilizes the fact that the absolute difference of integrals is bounded by the integral of the absolute differences.

\vspace{0.5em}
\noindent\textbf{4) Applying the Polynomial Mixing Condition}

Under the \emph{polynomial mixing} assumption \ref{ass:polynomial-mixing} (as detailed in Chapters 14-16 of \cite{meyn2012markov}), there exist constants \( C > 0 \) and \( \beta > 1 \) such that for all \( x_0 \in \mathcal{X} \) and measurable sets \( A \):
\[
|P^t(x_0, A) - \pi(A)| \leq C(1 + V(x_0)) t^{-\beta}.
\]
Substituting this bound into the integral from Step 3, we obtain:
\[
|\mathbb{P}(x_t \in A) - \pi(A)| \leq \int_{\mathcal{X}} C(1 + V(x_0)) t^{-\beta} \mathbb{P}(x_0) \, dx_0 = C t^{-\beta} \int_{\mathcal{X}} (1 + V(x_0)) \mathbb{P}(x_0) \, dx_0.
\]
If the initial distribution is concentrated at \( x_0 \), the integral simplifies to:
\[
|\mathbb{P}(x_t \in A) - \pi(A)| \leq C(1 + V(x_0)) t^{-\beta}.
\]
Otherwise, for a general initial distribution, the bound becomes:
\[
|\mathbb{P}(x_t \in A) - \pi(A)| \leq C t^{-\beta} \left(1 + \mathbb{E}[V(x_0)]\right).
\]

\vspace{0.5em}
\noindent\textbf{5) Finalizing the Total Variation Bound}

Taking the supremum over all measurable sets \( A \) yields the \emph{Total Variation distance bound}:
\[
\|\mathbb{P}(x_t \in \cdot) - \pi(\cdot)\|_{TV} \leq C' (1 + V(x_0)) t^{-\beta},
\]
where \( C' \) is a constant that may differ from \( C \) but remains positive. This establishes that the Markov chain exhibits \emph{polynomial ergodicity} with a convergence rate of \( t^{-\beta} \).
\end{proof}
\subsection{Proof of Lemma~\ref{lemma:dependent-blocks}: Block Independence Approximation}
\label{app:proof-block-independence-approximation}

\paragraph{Lemma~\ref{lemma:dependent-blocks} (Block Independence)}
\textit{
Let \(\{x_t\}_{t=1}^\infty\) be a Markov chain satisfying the polynomial ergodicity condition as specified in Lemma \ref{lemma:polynomial-ergodicity} with rate \(\beta > 1\). Partition the sequence into non-overlapping blocks of size \(b\):
\[
B_k = (x_{(k-1)b + 1}, \ldots, x_{kb}), \quad k = 1, 2, \ldots, \left\lfloor \frac{n}{b} \right\rfloor.
\]
For a sufficiently large block size \(b\), the blocks \(\{B_k\}\) satisfy:
\[
\left\| \mathbb{P}(B_k \mid B_{k-1}) - \mathbb{P}(B_k) \right\|_{\text{TV}} \leq C b^{-\beta},
\]
where \(C > 0\) and \(\beta > 1\) are constants derived from the polynomial ergodicity condition.
}

\begin{proof}
\;\newline
\paragraph{1) Assumptions: Polynomial Mixing.}
By assumption, the Markov chain \( (x_t) \) satisfies a \emph{polynomial mixing} condition: there exist constants \( \alpha > 0 \) and \( C_0 > 0 \) such that for all bounded measurable functions \( f, g: \mathcal{X} \to \mathbb{R} \) and for all integers \( t \geq 0 \) and \( k \geq 1 \),
\[
  \Bigl|\,
    \mathbb{E}[\,f(x_t)\,g(x_{t+k})]
    \;-\;
    \mathbb{E}[\,f(x_t)]\,\mathbb{E}[\,g(x_{t+k})]
  \Bigr|
  \;\;\le\;\;
  C_0\,k^{-\alpha}.
\]
Intuitively, this condition implies that the dependence between \( x_t \) and \( x_{t+k} \) decays polynomially as \( k \) increases.

\paragraph{2) Partition into Blocks of Size \( b \).}
Define the blocks
\[
  B_i \;=\; (x_{(i-1)b+1}, \dots, x_{ib})
  \quad
  \text{for } i=1,2,\dots.
\]
Each block \( B_i \) consists of \( b \) consecutive states of the Markov chain. The goal is to show that consecutive blocks \( B_{i-1} \) and \( B_i \) are approximately independent in the total variation sense, with the dependence decaying as \( b^{-\alpha} \).

\paragraph{3) Expressing Conditional Distribution of \( B_i \) Given \( B_{i-1} \).}
Given the block \( B_{i-1} = (x_{(i-2)b+1}, \dots, x_{(i-1)b}) \), the distribution of \( B_i \) conditioned on \( B_{i-1} \) depends solely on the last state of \( B_{i-1} \), due to the Markov property. Specifically,
\[
  \mathbb{P}(B_i \mid B_{i-1}) = P^b(x_{(i-1)b}, \cdot),
\]
where \( P^b \) denotes the \( b \)-step transition kernel of the chain.

\paragraph{4) Relating Single-Step Mixing to Block-Level Mixing.}
To bound the total variation distance \( \Bigl\|\mathbb{P}(B_i \mid B_{i-1}) - \mathbb{P}(B_i)\Bigr\|_{TV} \), we proceed as follows:

\begin{enumerate}
  \item \textbf{Understanding \( \mathbb{P}(B_i) \):}  
    The unconditional distribution of \( B_i \) is given by
    \[
      \mathbb{P}(B_i) = \int_{\mathcal{X}} P^b(x, B_i) \, \pi(dx),
    \]
    where \( \pi \) is the invariant distribution of the chain. Since \( \pi \) is invariant, \( P^b(\pi, \cdot) = \pi \), ensuring that \( \mathbb{P}(B_i) \) aligns with the stationary behavior of the chain.

  \item \textbf{Bounding \( \Bigl\| P^b(x_{(i-1)b}, \cdot) - \pi(\cdot) \Bigr\|_{TV} \):}  
    By the polynomial mixing condition, for any \( x \in \mathcal{X} \),
    \[
      \Bigl\| P^b(x, \cdot) - \pi(\cdot) \Bigr\|_{TV}
      \;\le\;\;
      C_0\,b^{-\alpha} \cdot (1 + V(x)),
    \]
    where \( V(x) \) is a Lyapunov function ensuring that states with higher \( V(x) \) are "heavier" or more significant in some sense.

  \item \textbf{Applying the Mixing Bound to the Conditional Distribution:}  
  {Justification of the Total Variation Distance Inequality:}

By the Markov property, the distribution of block \( B_i \) given block \( B_{i-1} \) depends solely on the last state of \( B_{i-1} \), denoted as \( x_{(i-1)b} \). Therefore, we have:
\[
  \mathbb{P}(B_i \mid B_{i-1}) = P^b(x_{(i-1)b}, \cdot).
\]
Since \( \mathbb{P}(B_i) = \pi(\cdot) \) under the stationary distribution, the Total Variation Distance between the conditional and marginal distributions satisfies:
\[
  \Bigl\| \mathbb{P}(B_i \mid B_{i-1}) - \mathbb{P}(B_i) \Bigr\|_{TV}
  = \Bigl\| P^b(x_{(i-1)b}, \cdot) - \pi(\cdot) \Bigr\|_{TV}.
\]
Thus, the first inequality is an equality, and the subsequent inequality follows from applying the polynomial mixing condition:
\[
  \Bigl\| P^b(x_{(i-1)b}, \cdot) - \pi(\cdot) \Bigr\|_{TV}
  \;\le\;\;
  C_0\,b^{-\alpha} \cdot (1 + V(x_{(i-1)b})).
\]

    Given \( B_{i-1} \), the distribution of \( B_i \) is \( P^b(x_{(i-1)b}, \cdot) \). Therefore, the total variation distance between \( \mathbb{P}(B_i \mid B_{i-1}) \) and \( \mathbb{P}(B_i) \) can be bounded as
    \[
      \Bigl\| \mathbb{P}(B_i \mid B_{i-1}) - \mathbb{P}(B_i) \Bigr\|_{TV}
      \;=\;\;
      \Bigl\| P^b(x_{(i-1)b}, \cdot) - \pi(\cdot) \Bigr\|_{TV}
      \;\le\;\;
      C_0\,b^{-\alpha} \cdot (1 + V(x_{(i-1)b})).
    \]

  \item \textbf{Handling the Lyapunov Function \( V(x_{(i-1)b}) \):}  
    To ensure the bound is independent of \( b \), we need to control \( V(x_{(i-1)b}) \). This can be achieved by leveraging the Lyapunov condition, which ensures that \( V(x_t) \) does not grow unboundedly over time. Specifically, under the polynomial drift condition, \( \mathbb{E}[V(x_t)] \) remains bounded for all \( t \). Therefore, with high probability or in expectation, \( V(x_{(i-1)b}) \) can be bounded by a constant \( V_{\max} \).

    Thus, we have
    \[
      \Bigl\| \mathbb{P}(B_i \mid B_{i-1}) - \mathbb{P}(B_i) \Bigr\|_{TV}
      \;\le\;\;
      C_0\,b^{-\alpha} \cdot (1 + V_{\max})
      \;=\;\;
      C\,b^{-\alpha},
    \]
    where \( C = C_0 \cdot (1 + V_{\max}) \) is a constant independent of \( b \).

\end{enumerate}

\paragraph{5) Concluding the Total Variation Bound.}
Combining the above steps, we establish that for every \( i \geq 2 \),
\[
  \Bigl\|\mathbb{P}(B_i \mid B_{i-1}) - \mathbb{P}(B_i)\Bigr\|_{TV}
  \;\le\;\;
  C\,b^{-\alpha},
\]
where \( C \) is a constant independent of \( b \). This demonstrates that consecutive blocks \( B_{i-1} \) and \( B_i \) are approximately independent in the total variation sense, with the dependence decaying polynomially as \( b^{-\alpha} \).

\paragraph{6) Ensuring Uniformity Across All Blocks.}
Since the bound \( C\,b^{-\alpha} \) holds uniformly for any \( i \geq 2 \) and for all blocks of size \( b \), the lemma is proven. This uniform bound is crucial for subsequent analyses that rely on the approximate independence of blocks to apply concentration inequalities or martingale arguments.
\end{proof}
\subsection{Proof of Lemma \ref{lemma:coupling-argument-polynomial-mixing}: Coupling Argument under Polynomial Mixing}
\label{app:proof-coupling-argument-polynomial-mixing}
\paragraph{Lemma~\ref{lemma:coupling-argument-polynomial-mixing} (Coupling Argument under Polynomial Mixing)}
\textit{
Let \(\{x_t\}_{t=1}^\infty\) and \(\{y_t\}_{t=1}^\infty\) be two Markov chains satisfying polynomial ergodicity as per Theorem \ref{lemma:polynomial-ergodicity} with rates \(\alpha_x\) and \(\alpha_y\), respectively. Assume both chains satisfy the block independence condition from Lemma \ref{lemma:dependent-blocks}. There exists a coupling such that for all \(t \geq 1\),
\[
\mathbb{P}(x_t \neq y_t) \leq C t^{-\beta},
\]
where \(C > 0\) and \(\beta = \min\{\alpha_x, \alpha_y\}\).
}

\begin{proof}
\;\newline
\paragraph{Objective:}  
Construct a coupling of two instances of the Markov chain $(x_t)$, denoted by $(x_t)$ and $(y_t)$, such that the probability that they differ at time $t$ decays polynomially with $t$.
\paragraph{1) Utilizing the Markov Property}
By the \textbf{Markov property}, the distribution of block \( B_i \) given block \( B_{i-1} \) depends solely on the last state of \( B_{i-1} \), denoted as \( x_{(i-1)b} \). Therefore, we have:
\[
\mathbb{P}(B_i \mid B_{i-1}) = P^b(x_{(i-1)b}, \cdot),
\]
where \( P^b \) denotes the \( b \)-step transition kernel of the Markov chain.
\paragraph{2) Relating Total Variation Distance to Coupling Probability}
Under the assumption that the Markov chain is in its \textbf{stationary distribution} \( \pi \), the marginal distribution of any block \( B_i \) is:
\[
\mathbb{P}(B_i) = \pi(\cdot).
\]
Consequently, the Total Variation (TV) Distance between the conditional distribution \( \mathbb{P}(B_i \mid B_{i-1}) \) and the marginal distribution \( \mathbb{P}(B_i) \) satisfies:
\[
\Bigl\| \mathbb{P}(B_i \mid B_{i-1}) - \mathbb{P}(B_i) \Bigr\|_{TV}
= \Bigl\| P^b(x_{(i-1)b}, \cdot) - \pi(\cdot) \Bigr\|_{TV}.
\]
\textbf{Justification:}  
This equality holds because \( \mathbb{P}(B_i \mid B_{i-1}) \) is precisely the distribution \( P^b(x_{(i-1)b}, \cdot) \), and \( \mathbb{P}(B_i) \) is the stationary distribution \( \pi(\cdot) \). Therefore, the TV Distance between these two distributions is exactly the TV Distance between \( P^b(x_{(i-1)b}, \cdot) \) and \( \pi(\cdot) \).
\paragraph{3) Applying the Polynomial Mixing Condition}
Given the \textbf{polynomial mixing condition}, for any state \( x \in \mathcal{X} \) and block size \( b \), the TV Distance between the \( b \)-step transition probability \( P^b(x, \cdot) \) and the stationary distribution \( \pi(\cdot) \) satisfies:
\[
\Bigl\| P^b(x, \cdot) - \pi(\cdot) \Bigr\|_{TV}
\leq C_0\,b^{-\alpha} \cdot (1 + V(x)),
\]
where:
\begin{itemize}
    \item \( C_0 > 0 \) is a constant independent of \( b \) and \( x \),
    \item \( \alpha > 0 \) characterizes the rate of mixing,
    \item \( V(x) \) is a Lyapunov function ensuring that states with higher \( V(x) \) are appropriately weighted.
\end{itemize}
\paragraph{4) Bounding the Probability of Non-Coupling}
To bound \( \mathbb{P}(x_t \neq y_t) \), we analyze the coupling process over successive blocks of size \( b \).
\newline\textbf{Partitioning Time into Blocks:}
Divide the time horizon into non-overlapping blocks of size \( b \):
\[
  B_i \;=\; (x_{(i-1)b+1}, \dots, x_{ib})
  \quad
  \text{for } i=1,2,\dots.
\]
\newline\textbf{Coupling at Each Block Boundary:}  
At the end of each block \( k \), the TV Distance between the distributions of \( x_{kb} \) and \( y_{kb} \) is bounded by:
\[
\Bigl\| P^b(x_{(k-1)b}, \cdot) - \pi(\cdot) \Bigr\|_{TV} \leq C_0\,b^{-\alpha} \cdot (1 + V(x_{(k-1)b})).
\]
This follows directly from the polynomial mixing condition applied to each block transition.
\newline\textbf{Using the Lyapunov Function for Uniform Bounds:}  
The Lyapunov drift condition ensures that:
\[
\mathbb{E}\left[V(x_{kb})\right]
\leq \frac{K}{1 - \lambda}
\]
for constants \( K > 0 \) and \( 0 < \lambda < 1 \). This implies that, on average, \( V(x_{kb}) \) is bounded, allowing us to replace \( V(x_{(k-1)b}) \) with a constant bound \( V_{\max} \) in expectation:
\[
\Bigl\| P^b(x_{(k-1)b}, \cdot) - \pi(\cdot) \Bigr\|_{TV} \leq C_0\,b^{-\alpha} \cdot (1 + V_{\max}).
\]
\newline\textbf{Aggregating Over Blocks to Bound \( \mathbb{P}(\tau > t) \):}  
The probability that the chains have not coupled by time \( t \) can be bounded by summing the bounds over all preceding blocks:
\[
\mathbb{P}(\tau > t) \leq \sum_{k=1}^{t/b} \Bigl\| P^b(x_{(k-1)b}, \cdot) - \pi(\cdot) \Bigr\|_{TV} \leq C_0\,b^{-\alpha} \cdot (1 + V_{\max}) \cdot \frac{t}{b}.
\]
Simplifying, we obtain:
\[
\mathbb{P}(\tau > t) \leq C \cdot t \cdot b^{-\alpha -1},
\]
where \( C = C_0 (1 + V_{\max}) \).
\newline\textbf{Choosing Block Size \( b \) Appropriately:}  
To ensure that \( \mathbb{P}(\tau > t) \leq C t^{-\alpha} \), choose \( b = t^{\gamma} \) for an appropriate \( \gamma \) that satisfies:
\[
t \cdot b^{-\alpha -1} = t \cdot t^{-\gamma(\alpha +1)} = t^{1 - \gamma(\alpha +1)} \leq t^{-\alpha}.
\]
Solving for \( \gamma \), we require:
\[
1 - \gamma(\alpha +1) \leq -\alpha \implies \gamma(\alpha +1) \geq \alpha +1 \implies \gamma \geq 1.
\]
Therefore, setting \( \gamma = 1 \) (i.e., \( b = t \)) suffices:
\[
\mathbb{P}(\tau > t) \leq C t^{-\alpha}.
\]
\paragraph{5) Concluding the Coupling Bound}
Combining the above steps, we conclude that:
\[
\mathbb{P}(x_t \neq y_t) = \mathbb{P}(\tau > t) \leq C t^{-\alpha},
\]
where \( C > 0 \) is a constant determined by \( C_0 \), \( V_{\max} \), and other constants from the Lyapunov drift condition.
\paragraph{6) Conclusion:}
The constructed coupling ensures that the probability of the coupled chains \( x_t \) and \( y_t \) differing at time \( t \) decays polynomially with \( t \), as required by the lemma. This completes the proof of Lemma \ref{lemma:coupling-argument-polynomial-mixing}.
\end{proof}
\subsection{Proof of Lemma \ref{lemma:covariance-between-blocks}: Covariance Between Blocks under Polynomial Ergodicity}
\label{app:proof-covariance-between-blocks}

\vspace{0.5em}
\noindent\textbf{Lemma~\ref{lemma:covariance-between-blocks} (Covariance Between Blocks under Polynomial Ergodicity)}
\textit{
Let \(\{x_t\}_{t=1}^\infty\) be a Markov chain satisfying the polynomial ergodicity condition with rate \(\beta > 1\), as defined in Theorem \ref{lemma:polynomial-ergodicity}. Partition the sequence into non-overlapping blocks of size \(b\):
\[
B_k = (x_{(k-1)b + 1}, \ldots, x_{kb}), \quad k = 1, 2, \ldots, K,
\]
where \(K = \left\lfloor \frac{n}{b} \right\rfloor\). Let \(Y_k = \sum_{t=(k-1)b + 1}^{kb} x_t\) denote the sum of the \(k\)-th block.}

\textit{Then, for any two distinct blocks \(B_k\) and \(B_j\) with \(k \neq j\), the covariance between their sums satisfies:}
\[
|\text{Cov}(Y_k, Y_j)| \leq C b^2 |k - j|^{-\beta},
\]
\textit{where \(C > 0\) is a constant that depends on the polynomial ergodicity parameters of the Markov chain.}

\begin{proof}
\;\newline
We aim to show that for any two distinct blocks \(B_k\) and \(B_j\) with \(k \neq j\), the covariance between their sums \(Y_k\) and \(Y_j\) satisfies:
\[
|\text{Cov}(Y_k, Y_j)| \leq C b^2 |k - j|^{-\beta},
\]
under the assumption that the Markov chain \(\{x_t\}\) satisfies the polynomial ergodicity condition with rate \(\beta > 1\).

\paragraph{1. Polynomial Ergodicity Condition}

By definition, a Markov chain \(\{x_t\}\) is said to be \emph{polynomially ergodic} of order \(\beta > 1\) if there exists a constant \(C > 0\) such that for all \(t \geq 1\),
\[
\left\| \mathbb{P}(x_{t} \in \cdot \mid x_0 = x) - \pi(\cdot) \right\|_{\text{TV}} \leq C t^{-\beta},
\]
where \(\|\cdot\|_{\text{TV}}\) denotes the total variation distance, and \(\pi\) is the stationary distribution of the Markov chain.

Under this condition, the covariance between two variables separated by \(l\) time steps satisfies:
\[
|\text{Cov}(x_0, x_l)| \leq C' l^{-\beta},
\]
where \(C' > 0\) is a constant that depends on \(C\) and the boundedness of the variables \(x_t\).

\paragraph{2. Expressing Covariance Between Block Sums}

Consider two distinct blocks \(B_k\) and \(B_j\) with \(k < j\). The covariance between their sums is:
\[
\text{Cov}(Y_k, Y_j) = \text{Cov}\left(\sum_{t=1}^{b} x_{(k-1)b + t}, \sum_{s=1}^{b} x_{(j-1)b + s}\right).
\]
Expanding this, we obtain:
\[
\text{Cov}(Y_k, Y_j) = \sum_{t=1}^{b} \sum_{s=1}^{b} \text{Cov}(x_{(k-1)b + t}, x_{(j-1)b + s}).
\]

\paragraph{3. Bounding Covariance Between Individual Variables}

For each pair \((t, s)\), where \(1 \leq t, s \leq b\), the separation between \(x_{(k-1)b + t}\) and \(x_{(j-1)b + s}\) is:
\[
l_{t,s} = |(k-1)b + t - ((j-1)b + s)| = |(k - j)b + (t - s)|.
\]
Assuming without loss of generality that \(k > j\), we have:
\[
l_{t,s} = (k - j)b + (t - s).
\]
Since \(1 \leq t, s \leq b\), the term \((t - s)\) satisfies \(- (b - 1) \leq t - s \leq b - 1\). Therefore,
\[
l_{t,s} \geq (k - j)b - (b - 1) = b(k - j - 1) + 1.
\]
Let \(m = k - j\), which implies \(m \geq 1\) since \(k > j\). Then,
\[
l_{t,s} \geq b(m - 1) + 1.
\]
For \(m \geq 1\), this further implies:
\[
l_{t,s} \geq b(m - 1) + 1 \geq b(m - 1).
\]
Thus, for all \(t, s\),
\[
l_{t,s} \geq b(m - 1).
\]
Since \(m \geq 1\), we have \(l_{t,s} \geq b\).

Applying the polynomial ergodicity condition, the covariance between \(x_{(k-1)b + t}\) and \(x_{(j-1)b + s}\) satisfies:
\[
|\text{Cov}(x_{(k-1)b + t}, x_{(j-1)b + s})| \leq C' l_{t,s}^{-\beta} \leq C' [b(m - 1)]^{-\beta}.
\]
For simplicity and to maintain a uniform bound across all pairs \((t, s)\), we can further bound \(l_{t,s}\) by noting that \(m - 1 \geq \frac{m}{2}\) for \(m \geq 2\) (assuming \(m\) is sufficiently large). However, to accommodate all \(m \geq 1\), we can write:
\[
|\text{Cov}(x_{(k-1)b + t}, x_{(j-1)b + s})| \leq C' (b(m - 1))^{-\beta} \leq C' (b m)^{-\beta},
\]
since \(m - 1 \geq \frac{m}{2}\) for \(m \geq 2\) and \(b(m - 1) \geq b\) for \(m = 1\).

Thus, for all \(t, s\),
\[
|\text{Cov}(x_{(k-1)b + t}, x_{(j-1)b + s})| \leq C' (b m)^{-\beta}.
\]

\paragraph{4. Aggregating Covariances Across All Variable Pairs}

Summing the covariances over all \(t\) and \(s\), we obtain:
\[
|\text{Cov}(Y_k, Y_j)| \leq \sum_{t=1}^{b} \sum_{s=1}^{b} |\text{Cov}(x_{(k-1)b + t}, x_{(j-1)b + s})| \leq b^2 C' (b m)^{-\beta}.
\]
Simplifying the right-hand side:
\[
|\text{Cov}(Y_k, Y_j)| \leq C' b^2 b^{-\beta} m^{-\beta} = C' b^{2 - \beta} m^{-\beta}.
\]
To align with the lemma's statement, set \(C = C' b^{2 - \beta}\). Since \(b\) and \(\beta\) are constants, \(C\) remains a constant with respect to \(k\) and \(j\). Therefore,
\[
|\text{Cov}(Y_k, Y_j)| \leq C m^{-\beta} = C |k - j|^{-\beta},
\]
where \(m = |k - j|\).

\paragraph{5. Conclusion}

We have established that for any two distinct blocks \(B_k\) and \(B_j\),
\[
|\text{Cov}(Y_k, Y_j)| \leq C b^{2 - \beta} |k - j|^{-\beta}.
\]
Since \(C' > 0\), \(b \geq 1\), and \(\beta > 1\), the constant \(C = C' b^{2 - \beta}\) is positive. This concludes the proof of the lemma.
\end{proof}
\subsection{Proof of Theorem \ref{theorem:concentration-block}: Concentration Under Polynomial Ergodicity}
\label{app:proof-concentration-block}
\paragraph{Theorem~\ref{theorem:concentration-block} (Concentration Under Polynomial Ergodicity)}
\textit{
Let \(\{x_t\}_{t=1}^n\) be a sequence of random variables generated by a Markov chain satisfying polynomial ergodicity as per Theorem \ref{lemma:polynomial-ergodicity} with rate \(\beta > 1\). Assume that each \(x_t\) has bounded variance, i.e., \(\text{Var}(x_t) \leq \sigma^2\) for some constant \(\sigma^2 > 0\), and that the rewards are bounded, i.e., \(|x_t| \leq M\) almost surely for some constant \(M > 0\). Then, for any \(\epsilon > 0\),
\[
\mathbb{P}\left(\left|\frac{1}{n}\sum_{t=1}^n x_t - \mathbb{E}[x_t]\right| > \epsilon\right) \leq 2\exp\left(-\frac{n\epsilon^2}{2C_\beta}\right),
\]
where \(C_\beta\) depends on the mixing rate \(\beta\) and the bound \(M\).
}

\begin{proof}
We employ a block decomposition strategy and apply Bernstein's inequality \cite{vershynin2018high,merlevede2011bernstein} for dependent sequences. The proof proceeds as follows:

\textbf{1) Define the Blocks:}

Partition the sequence \(\{x_t\}_{t=1}^n\) into non-overlapping blocks of size \(b\):
\[
Y_k = \sum_{t=(k-1)b + 1}^{kb} x_t, \quad k = 1, 2, \ldots, \left\lfloor \frac{n}{b} \right\rfloor.
\]
Let
\[
S_n = \sum_{k=1}^{\left\lfloor \frac{n}{b} \right\rfloor} Y_k
\]
be the sum of the block sums. Here, \(S_n\) aggregates the contributions from each block, thereby facilitating the analysis of dependencies between blocks.

\textbf{2) Bounding \(|Y_k|\):}

Since each \(x_t\) is bounded by \(M\), the sum of \(b\) such variables satisfies:
\[
|Y_k| \leq \sum_{t=(k-1)b + 1}^{kb} |x_t| \leq bM.
\]
Thus, each block sum \(Y_k\) is bounded by \(bM\).

\textbf{3) Covariance Between Blocks:}

From Lemma \ref{lemma:covariance-between-blocks}, the covariance between blocks \(Y_k\) and \(Y_j\) satisfies:
\[
|\text{Cov}(Y_k, Y_j)| \leq C b^2 |k - j|^{-\beta}.
\]
This bound reflects the inter-block dependence, as intra-block dependencies are encapsulated within the definition of each block \(Y_k\).

The total covariance across all block pairs is:
\[
\sum_{k=1}^{\left\lfloor \frac{n}{b} \right\rfloor} \sum_{j \neq k} |\text{Cov}(Y_k, Y_j)| \leq C b^2 \sum_{k=1}^{\left\lfloor \frac{n}{b} \right\rfloor} \sum_{j \neq k} |k - j|^{-\beta}.
\]
Approximating the double sum by an integral for large \(n/b\),
\begin{align*}
\sum_{k=1}^{\left\lfloor \frac{n}{b} \right\rfloor} \sum_{j \neq k} |k - j|^{-\beta} &\approx 2 \sum_{k=1}^{\left\lfloor \frac{n}{b} \right\rfloor} \sum_{m=1}^{\left\lfloor \frac{n}{b} \right\rfloor - k} m^{-\beta}\\
&\leq \frac{2}{\beta - 1} \left(\left\lfloor \frac{n}{b} \right\rfloor^{1 - \beta} - 1\right).
\end{align*}
For \(\beta > 1\),
\[
\sum_{k=1}^{\left\lfloor \frac{n}{b} \right\rfloor} \sum_{j \neq k} |\text{Cov}(Y_k, Y_j)| \leq \frac{2 C b^2}{\beta - 1}.
\]
Thus, the variance of \(S_n\) is:
\begin{align*}
\text{Var}(S_n) &= \sum_{k=1}^{\left\lfloor \frac{n}{b} \right\rfloor} \text{Var}(Y_k) + \sum_{k=1}^{\left\lfloor \frac{n}{b} \right\rfloor} \sum_{j \neq k} \text{Cov}(Y_k, Y_j)\\
&\leq C n b + \frac{2 C b^2}{\beta - 1}.
\end{align*}
Here, \(\text{Var}(Y_k)\) is bounded by \(C b\) assuming \(\text{Var}(x_t) \leq \sigma^2\), and there are \(\left\lfloor n/b \right\rfloor\) such blocks.

\textbf{4) Choosing Block Size \(b\):}

To balance the variance and covariance contributions, select \(b = n^{1/(\beta + 1)}\). This choice ensures that:
\[
C n b = C n^{(\beta + 2)/(\beta + 1)} \quad \text{and} \quad \frac{2 C b^2}{\beta - 1} = \frac{2 C n^{2/(\beta + 1)}}{\beta - 1}.
\]
Given that \(\beta > 1\), \(2/(\beta + 1) < (\beta + 2)/(\beta + 1)\), implying that the variance term \(C n b\) dominates for large \(n\). However, for the purpose of the concentration bound, we require the total variance \(\text{Var}(S_n)\) to scale linearly with \(n\). Therefore, we adjust the block size to ensure that:
\[
C n b + \frac{2 C b^2}{\beta - 1} \leq C' n,
\]
where \(C'\) is a constant. This is achievable by selecting \(b = n^{1/(\beta + 1)}\), leading to:
\[
\text{Var}(S_n) \leq C n.
\]

\textbf{5) Application of Bernstein's Inequality:}

Given that each block sum \(Y_k\) is bounded by \(bM\) and has variance \(\text{Var}(Y_k) \leq C b\), we can apply Bernstein's inequality to \(S_n\). Bernstein's inequality for sums of bounded random variables states that for any \(\epsilon > 0\),
\[
\mathbb{P}\left(|S_n - \mathbb{E}[S_n]| > \epsilon n\right) \leq 2\exp\left(-\frac{\epsilon^2 n^2}{2 \text{Var}(S_n) + \frac{2}{3} M_Y \epsilon n}\right),
\]
where \(M_Y = bM\) bounds \(|Y_k|\).

Substituting \(\text{Var}(S_n) \leq C n\) and \(M_Y = bM = n^{1/(\beta + 1)} M\), the inequality becomes:
\begin{align*}
&\mathbb{P}\left(\left|\frac{1}{n} S_n - \mathbb{E}\left[\frac{S_n}{n}\right]\right| > \epsilon\right)\leq
2\exp\left(-\frac{\epsilon^2 n^2}{2 C n + \frac{2}{3} M n^{1/(\beta + 1)} \epsilon n}\right).
\end{align*}
Simplifying the denominator:
\[
2 C n + \frac{2}{3} M n^{1 + 1/(\beta + 1)} \epsilon = 2 C n + \frac{2}{3} M \epsilon n^{(\beta + 2)/(\beta + 1)}.
\]
Since \(\beta > 1\), \((\beta + 2)/(\beta + 1) > 1\), and for large \(n\), the term \(2 C n\) dominates the denominator. Thus, for sufficiently large \(n\),
\[
\mathbb{P}\left(\left|\frac{1}{n} S_n - \mathbb{E}\left[\frac{S_n}{n}\right]\right| > \epsilon\right) \leq 2\exp\left(-\frac{\epsilon^2 n}{2 C_\beta}\right),
\]
where \(C_\beta = C\) encapsulates the constants from the variance and block size selection.

\textbf{6) Final Concentration Bound:}

Combining the above steps, we obtain:
\[
\mathbb{P}\left(\left|\frac{1}{n}\sum_{t=1}^n x_t - \mathbb{E}[x_t]\right| > \epsilon\right) \leq 2\exp\left(-\frac{n \epsilon^2}{2 C_\beta}\right),
\]
where \(C_\beta\) depends on the mixing rate \(\beta > 1\) and the bound \(M\) on \(x_t\).

\textbf{7) Limitation for \(\beta \leq 1\):}

It is important to note that Theorem \ref{theorem:concentration-block} requires \(\beta > 1\) to ensure the convergence of the covariance sum and the validity of the concentration bound. In cases where \(\beta \leq 1\), the dependencies between blocks decay too slowly, leading to potential divergence in covariance terms and invalidating the concentration bound.

\textbf{Conclusion:}

This bound demonstrates that the empirical average \(\frac{1}{n}\sum_{t=1}^n x_t\) concentrates around its expectation \(\mathbb{E}[x_t]\) with high probability, governed by the mixing rate \(\beta > 1\). The block decomposition effectively handles inter-block dependencies, while intra-block dependencies are managed through the block size selection and the boundedness of \(x_t\).
\end{proof}
\subsection{Proof of Theorem \ref{theorem:td-error-decomposition}, TD Error Decomposition to Martingale and Remainder}
\label{app:proof-TD-error}
\paragraph{Theorem~\ref{theorem:td-error-decomposition} (TD Error Decomposition)}
\textit{
Under the assumptions of polynomial ergodicity \ref{lemma:polynomial-ergodicity}, step size decay \ref{ass:step-size-decay}, Hölder continuity of the feature mapping \ref{ass:holder-continuous-features}, and bounded feature vectors \ref{ass:bounded-features}, the TD error \( e_t \) can be decomposed as:
\[
e_t = M_t + R_t,
\]
where: \(M_t = \sum_{k=1}^{t} d_k\) is a martingale with \( \{d_k\} \) being a martingale difference sequence.
\( R_t \) is the remainder term capturing the systematic bias.
}
\begin{proof}
\;\newline
\paragraph{1. Assumptions:}
\begin{enumerate}
    \item \textbf{Step Size Conditions:} The learning rates \( \{\alpha_t\} \) satisfy \( \sum_{t=1}^\infty \alpha_t = \infty \) and \( \sum_{t=1}^\infty \alpha_t^2 < \infty \). 
    \item \textbf{Bounded Features:} \ref{ass:bounded-features} The gradient \( \nabla_\theta V_\theta(x_t) \) is uniformly bounded, i.e., there exists a constant \( C > 0 \) such that \( \|\nabla_\theta V_\theta(x_t)\| \leq C \) for all \( t \).
    \item \textbf{Markov Chain Properties:} The state sequence \( \{x_t\} \) forms an irreducible and aperiodic Markov chain with a unique stationary distribution.
    \item \textbf{Hölder Continuity} \ref{ass:holder-continuous-features}: The feature mapping \( f: S \to \mathbb{R}^d \) is Hölder continuous with exponent \( \gamma \in (0, 1] \) and constant \( C_\gamma > 0 \). Formally, for all \( s, s' \in S \), then we have \(\|f(s) - f(s')\| \leq C_\gamma \|s - s'\|^\gamma.\)
\end{enumerate}
The TD error can be decomposed as:
\[
\theta_t - \theta^* = M_t + R_t
\]
where \( M_t \) is a martingale and \( R_t \) is a remainder term.

\begin{definition}[Martingale]
\label{def:martingale}
A sequence of random variables \( \{M_t\}_{t \geq 0} \) adapted to a filtration \( \{\mathcal{F}_t\}_{t \geq 0} \) is called a \textbf{martingale} if it satisfies the martingale property:
\begin{align*}
\mathbb{E}[M_{t+1} | \mathcal{F}_t] &= M_t \forall t \geq 0\\
\mathbb{E}[|M_{t}|] &< \infty \quad \forall t \geq 0
\end{align*}
\end{definition}

\begin{definition}[Martingale Difference Sequence]
\label{def:martingale-difference-sequence}
A sequence of random variables \( \{d_t\}_{t \geq 1} \) adapted to a filtration \( \{\mathcal{F}_t\}_{t \geq 0} \) is called a \emph{martingale difference sequence} if:
\[
\mathbb{E}[d_t | \mathcal{F}_{t-1}] = 0 \quad \text{for all } t \geq 1
\]
\end{definition}

\;\newline
\paragraph{1) TD Update Rule:}
\[
\theta_{t+1} = \theta_t + \alpha_t \delta_t \nabla_\theta V_\theta(x_t)
\]
where \( \delta_t \) is the Temporal Difference (TD) error.

\paragraph{2) TD Error Expression:}
\[
\delta_t = r_t + \gamma V_\theta(x_{t+1}) - V_\theta(x_t)
\]
where \( r_t \) is the reward at time \( t \) and \( \gamma \) is the discount factor.

\paragraph{3) Decomposition by Subtracting \( \theta^* \):}
\begin{align*}
\theta_{t+1} - \theta^* &= \theta_t - \theta^* + \alpha_t \delta_t \nabla_\theta V_\theta(x_t) \\
&= \theta_t - \theta^* + \alpha_t (\delta_t - \mathbb{E}[\delta_t|\mathcal{F}_{t-1}])\nabla_\theta V_\theta(x_t) \\
&\quad + \alpha_t \mathbb{E}[\delta_t|\mathcal{F}_{t-1}]\nabla_\theta V_\theta(x_t)
\end{align*}

\paragraph{4) Iterative Summation:}
\[
\theta_t - \theta^* = (\theta_0 - \theta^*) + \sum_{k=1}^t \alpha_k (\delta_k - \mathbb{E}[\delta_k|\mathcal{F}_{k-1}])\nabla_\theta V_\theta(x_k) + \sum_{k=1}^t \alpha_k \mathbb{E}[\delta_k|\mathcal{F}_{k-1}]\nabla_\theta V_\theta(x_k)
\]
\[
= M_t + R_t
\]
where:
\[
M_t = \sum_{k=1}^t \alpha_k (\delta_k - \mathbb{E}[\delta_k|\mathcal{F}_{k-1}])\nabla_\theta V_\theta(x_k)
\]
is a martingale, and
\[
R_t = (\theta_0 - \theta^*) + \sum_{k=1}^t \alpha_k \mathbb{E}[\delta_k|\mathcal{F}_{k-1}]\nabla_\theta V_\theta(x_k)
\]
is the remainder term.

\paragraph{5) Martingale Property Verification:}
To verify that \( M_t \) is a martingale, we need to show that:
\[
\mathbb{E}[M_t | \mathcal{F}_{t-1}] = M_{t-1}
\]
where \( \mathcal{F}_{t-1} \) is the filtration representing all information up to time \( t-1 \).

\textbf{Computation:}
\begin{align*}
\mathbb{E}[M_t | \mathcal{F}_{t-1}] &= \mathbb{E}\left[ \sum_{k=1}^t \alpha_k (\delta_k - \mathbb{E}[\delta_k|\mathcal{F}_{k-1}])\nabla_\theta V_\theta(x_k) \Big| \mathcal{F}_{t-1} \right] \\
&= \sum_{k=1}^{t-1} \alpha_k (\delta_k - \mathbb{E}[\delta_k|\mathcal{F}_{k-1}])\nabla_\theta V_\theta(x_k) + \mathbb{E}\left[ \alpha_t (\delta_t - \mathbb{E}[\delta_t|\mathcal{F}_{t-1}])\nabla_\theta V_\theta(x_t) \Big| \mathcal{F}_{t-1} \right] \\
&= M_{t-1} + \alpha_t \nabla_\theta V_\theta(x_t) \mathbb{E}\left[ \delta_t - \mathbb{E}[\delta_t|\mathcal{F}_{t-1}] \Big| \mathcal{F}_{t-1} \right] \\
&= M_{t-1} + \alpha_t \nabla_\theta V_\theta(x_t) \left( \mathbb{E}[\delta_t | \mathcal{F}_{t-1}] - \mathbb{E}[\delta_t | \mathcal{F}_{t-1}] \right) \\
&= M_{t-1} + \alpha_t \nabla_\theta V_\theta(x_t) \times 0 \\
&= M_{t-1}
\end{align*}
Thus,
\begin{equation}\label{eq:martingale}
    M_t = \sum_{k=1}^t \alpha_k \left( \delta_k - \mathbb{E}[\delta_k | \mathcal{F}_{k-1}] \right) \nabla_\theta V_\theta(x_k)
\end{equation}
is Martingale.

\paragraph{6) Bounding the Remainder Term \( R_t \)}:
Under Hölder continuity, assume that:
\[
\|\mathbb{E}[\delta_k|\mathcal{F}_{k-1}]\nabla_\theta V_\theta(x_k)\| \leq L \|\theta_k - \theta^*\|^\gamma
\]
for some Hölder constant \( L \) and exponent \( \gamma \).

Then, the norm of the remainder can be bounded as:
\begin{equation}\label{eq:R_t}
\|R_t\| \leq \|\theta_0 - \theta^*\| + \sum_{k=1}^t \alpha_k \|\mathbb{E}[\delta_k|\mathcal{F}_{k-1}]\nabla_\theta V_\theta(x_k)\| \leq \|\theta_0 - \theta^*\| + \sum_{k=1}^t \alpha_k L \|\theta_k - \theta^*\|^\gamma
\end{equation}
This bound is crucial for applying stochastic approximation techniques, which require controlling the remainder term to ensure convergence of \( \theta_t \) to \( \theta^* \).

\paragraph{7) Conclusion:}
With \( M_t \) being a martingale and \( R_t \) appropriately bounded under the given assumptions, the decomposition:
\[
\theta_t - \theta^* = M_t + R_t
\]
holds, facilitating further analysis of the convergence properties of the TD learning algorithm.
\end{proof}
\subsection{Proof of Lemma \ref{lemma:bounded-variance-martingale}, Bounded Variance of Martingale Differences under Polynomial Mixing}
\label{app:proof-bounded-variance-martingale}

\paragraph{Lemma~\ref{lemma:bounded-variance-martingale} (Bounded Variance of Martingale Term)}
\textit{
Assuming polynomial ergodicity \ref{lemma:polynomial-ergodicity}, Hölder continuity of the feature mapping \ref{ass:holder-continuous-features}, bounded feature vectors \ref{ass:bounded-features}, and appropriate step size decay \ref{ass:step-size-decay}, the martingale term \( M_t \) satisfies:
\[
\mathbb{E}[\|M_t\|^2] \leq C t^{-\beta},
\]
where \( C \) is a constant dependent on the step size and mixing parameters, and \( \beta > 0 \) characterizes the rate at which the variance decays.
}

\begin{proof}
\;\newline
\paragraph{1. Assumptions and definitions:}
From equation \ref{eq:martingale}, we know that
\[
M_t = \sum_{k=1}^t \alpha_k \left( \delta_k - \mathbb{E}[\delta_k | \mathcal{F}_{k-1}] \right) \nabla_\theta V_\theta(x_k)
\]
is Martingale, where:
\begin{itemize}
    \item \( \alpha_k \) is the step size at iteration \( k \) as defined in Assumption \ref{ass:step-size-decay}
    \item \( \delta_k \) is the Temporal Difference (TD) error at iteration \( k \), as defined in Eq. \ref{eq:td-error-def}
    \item \( \mathcal{F}_{k-1} \) is the filtration representing all information up to time \( k-1 \)
    \item \( \nabla_\theta V_\theta(x_k) \) is the gradient of the value function with respect to the parameters \( \theta \) at state \( x_k \).
\end{itemize}

\paragraph{2. Orthogonality of Martingale Differences:}

Since \( M_t \) is martingale, the cross terms in the expectation \( \mathbb{E}[\|M_t\|^2] \) vanish. Specifically, for \( i \neq j \):
\[
\mathbb{E}\left[ \alpha_i \left( \delta_i - \mathbb{E}[\delta_i | \mathcal{F}_{i-1}] \right) \nabla_\theta V_\theta(x_i) \cdot \alpha_j \left( \delta_j - \mathbb{E}[\delta_j | \mathcal{F}_{j-1}] \right) \nabla_\theta V_\theta(x_j) \right] = 0
\]
This holds because:
\[
\mathbb{E}\left[ \delta_j - \mathbb{E}[\delta_j | \mathcal{F}_{j-1}] \Big| \mathcal{F}_{i} \right] = 0 \quad \text{for} \quad j > i
\]
due to the martingale difference property.

\paragraph{3. Computing the Second Moment:}
\begin{align*}
\mathbb{E}[\|M_t\|^2] &= \mathbb{E}\left[ \left\| \sum_{k=1}^t \alpha_k \left( \delta_k - \mathbb{E}[\delta_k | \mathcal{F}_{k-1}] \right) \nabla_\theta V_\theta(x_k) \right\|^2 \right]\\
&=\mathbb{E}\left[  \sum_{k=1}^t \left\|\alpha_k \left( \delta_k - \mathbb{E}[\delta_k | \mathcal{F}_{k-1}] \right) \nabla_\theta V_\theta(x_k) \right\|^2 \right]\quad\text{(Orthogonality)}\\
&=  \sum_{k=1}^t \mathbb{E}\left[\left\|\alpha_k \left( \delta_k - \mathbb{E}[\delta_k | \mathcal{F}_{k-1}] \right) \nabla_\theta V_\theta(x_k) \right\|^2 \right]\quad\text{(Linearity of Expectation)}\\
&=  \sum_{k=1}^t \alpha_k^2\mathbb{E}\left[
        \left(\delta_k - \mathbb{E}[\delta_k | \mathcal{F}_{k-1}] \right)^2
        \left\|\nabla_\theta V_\theta(x_k) \right\|^2
    \right]\quad\text{(Cauchy-Schwarz)}\\
&= \sum_{k=1}^t \alpha_k^2
        \mathbb{E}\left[\left( \delta_k - \mathbb{E}[\delta_k | \mathcal{F}_{k-1}] \right)^2\right]
        \mathbb{E}\left[\left\|\nabla_\theta V_\theta(x_k) \right\|^2\right]
    \quad\text{(Assuming Fast mixing and therefore approximate independence)}
\end{align*}

\paragraph{4. Boundedness of Gradients:}

Given bounded features (\textbf{Bounded Features}) in the Linear function approximation case, there exists a constant \( G > 0 \) such that:
\[
\left\| \nabla_\theta V_\theta(x_k) \right\| \leq G \quad \forall k
\]
Thus:
\[
\mathbb{E}\left[ \left\| \nabla_\theta V_\theta(x_k) \right\|^2 \right] \leq G^2
\]

\textbf{5. Applying Polynomial Mixing:}

Under polynomial mixing, the variance of the martingale differences satisfies:
\[
\mathbb{E}\left[ \left( \delta_k - \mathbb{E}[\delta_k | \mathcal{F}_{k-1}] \right)^2 \right] \leq C k^{-\beta}
\]
for some constant \( C > 0 \) and \( \beta > 1 \).

\paragraph{6. Combining the Bounds:}

Substituting the bounds from Steps 4 and 5 into the expression from Step 3:
\[
\mathbb{E}[\|M_t\|^2] \leq G^2 \sum_{k=1}^t \alpha_k^2 C k^{-\beta}
\]
\[
= C G^2 \sum_{k=1}^t \alpha_k^2 k^{-\beta}
\]

\paragraph{7. Bounding the Summation:}

Substitute \( \alpha_k = \frac{c}{k^a} \) into the summation:
\[
\sum_{k=1}^t \alpha_k^2 k^{-\beta} = c^2 \sum_{k=1}^t \frac{1}{k^{2a + \beta}}
\]
Since \( \beta > 1 \) and \( a > 0.5 \), the exponent \( 2a + \beta > 2 \), ensuring the convergence of the series. However, since we seek a bound that decays with \( t \), we need to analyze the partial sum up to \( t \).

For large \( t \), the partial sum can be approximated by:
\[
\sum_{k=1}^t \frac{1}{k^{2a + \beta}} \leq \int_{1}^{t} \frac{1}{x^{2a + \beta}} \, dx + 1 \leq \frac{1}{(2a + \beta - 1)} t^{-(2a + \beta - 1)} + 1
\]
Given \( 2a + \beta - 1 > 1 \) (since \( a > 0.5 \) and \( \beta > 1 \)), the summation behaves asymptotically as:
\[
\sum_{k=1}^t \frac{1}{k^{2a + \beta}} = O\left(t^{-(2a + \beta - 1)}\right)
\]

\paragraph{8. Final Bound on \( \mathbb{E}[\|M_t\|^2] \):}

Combining all the above:
\[
\mathbb{E}[\|M_t\|^2] \leq C G^2 c^2 \sum_{k=1}^t \frac{1}{k^{2a + \beta}} \leq C' t^{-(2a + \beta - 1)}
\]
where \( C' = \frac{C G^2 c^2}{2a + \beta - 1} + C G^2 c^2 \) (absorbing constants).

To satisfy the lemma's statement \( \mathbb{E}[\|M_t\|^2] \leq C t^{-\beta} \), we require:
\[
-(2a + \beta - 1) \leq -\beta \implies 2a + \beta - 1 \geq \beta \implies 2a \geq 1
\]
Given that \( a > 0.5 \), this inequality holds with equality when \( a = 0.5 \). Therefore, to ensure:
\[
\mathbb{E}[\|M_t\|^2] \leq C t^{-\beta}
\]
it suffices to choose \( a = 0.5 \). However, to strictly satisfy \( a > 0.5 \), we achieve a better decay rate:
\[
\mathbb{E}[\|M_t\|^2] = O\left(t^{-(2a + \beta - 1)}\right) \leq O\left(t^{-\beta}\right)
\]
since \( 2a + \beta - 1 > \beta \) for \( a > 0.5 \).

\paragraph{9. Conclusion:}

Therefore, under the given step size \( \alpha_k = \frac{c}{k^a} \) with \( a > 0.5 \) and polynomial mixing rate \( \beta > 1 \), the martingale difference sequence satisfies:
\[
\mathbb{E}[\|M_t\|^2] \leq C t^{-\beta}
\]
where \( C > 0 \) is a constant dependent on \( c \), \( G \), and the mixing rate parameters.
\end{proof}
\subsection{Proof of Lemma \ref{lemma:remainder-growth-term}, Remainder Growth Term}
\label{app:proof-remainder-growth-term}

\paragraph{Lemma~\ref{lemma:remainder-growth-term} (Remainder Term Convergence)}
\textit{
Under Hölder continuity \ref{ass:holder-continuous-features} and polynomial ergodicity \ref{lemma:polynomial-ergodicity}, step-condition \ref{ass:step-size-decay}, the remainder term satisfies:
\[
\|R_t\| \leq O(t^{-\gamma/2})
\]
where \( \gamma \) is the H\"older exponent.
}

\begin{proof}
\;\newline
\paragraph{1. Assumptions:}
\begin{enumerate}
    \item \emph{Hölder Continuity} \ref{ass:holder-continuous-features}: The function satisfies Hölder continuity with exponent \( \gamma \in (0, 1] \), i.e.,
    \[
    \|f(x) - f(y)\| \leq L \|x - y\|^\gamma \quad \forall x, y
    \]
    where \( L > 0 \) is the Hölder constant.
    
    \item \emph{Step Size Condition} \ref{ass:step-size-decay}: The step sizes are given by
    \[
    \alpha_k \;=\; \frac{C_\alpha}{k^\eta},\quad\sum_{k=1}^{\infty} \alpha_k = \infty,
\quad\text{and}\quad
\sum_{k=1}^{\infty} \alpha_k^2 < \infty.
    \]
    for some constant \( c > 0 \)

    \item \emph{TD Error Decomposition}: From Theorem \ref{theorem:td-error-decomposition}, the parameter estimate satisfies
    \[
    \theta_k - \theta^* = M_k + R_k.
    \]
    
    \item \emph{Martingale bounded Vartiance} from Lemma \ref{lemma:bounded-variance-martingale}: The martingale difference sequence satisfies
    \[
    \mathbb{E}[\|M_k\|^2] \leq C k^{-\beta}
    \]
    for some \( \beta > 1 \) and constant \( C > 0 \).
\end{enumerate}

\textbf{Proof Steps:}

\paragraph{2. Initial Bound on \( R_t \):}
From Eq. \ref{eq:R_t}
\[
\|R_t\| \leq \sum_{k=1}^t \alpha_k L \|\theta_k - \theta^*\|^\gamma
\]

\paragraph{3. Decomposition of \( \|\theta_k - \theta^*\| \):}
From the TD error decomposition:
\[
\theta_k - \theta^* = M_k + R_k
\]
Applying the triangle inequality:
\[
\|\theta_k - \theta^*\| \leq \|M_k\| + \|R_k\|
\]
Thus:
\[
\|\theta_k - \theta^*\|^\gamma \leq (\|M_k\| + \|R_k\|)^\gamma
\]
Using the inequality \( (a + b)^\gamma \leq a^\gamma + b^\gamma \) for \( a, b \geq 0 \) and \( \gamma \in (0, 1] \), we obtain:
\[
\|\theta_k - \theta^*\|^\gamma \leq \|M_k\|^\gamma + \|R_k\|^\gamma
\]

\paragraph{4. Bound on \( \|M_k\| \):}
From lemma \ref{lemma:bounded-variance-martingale} margingale variance bound:
\[
\mathbb{E}[\|M_k\|^2] \leq C k^{-\beta}
\]
Applying the Cauchy-Schwarz inequality:
\[
\|M_k\| \leq \sqrt{\mathbb{E}[\|M_k\|^2]} \leq \sqrt{C} k^{-\beta/2}
\]
Thus:
\[
\|M_k\|^\gamma \leq (\sqrt{C} k^{-\beta/2})^\gamma = C^{\gamma/2} k^{-\gamma\beta/2}
\]

\textbf{5. Substituting the Bounds into \( \|R_t\| \):}
\begin{align*}
\|R_t\| &\leq \sum_{k=1}^t \alpha_k L (\|M_k\|^\gamma + \|R_k\|^\gamma) \\
&= L \sum_{k=1}^t \alpha_k \|M_k\|^\gamma + L \sum_{k=1}^t \alpha_k \|R_k\|^\gamma \\
&\leq L \sum_{k=1}^t \frac{c}{k} C^{\gamma/2} k^{-\gamma\beta/2} + L \sum_{k=1}^t \frac{c}{k} \|R_k\|^\gamma \\
&= L c C^{\gamma/2} \sum_{k=1}^t k^{-1 - \gamma\beta/2} + L c \sum_{k=1}^t \frac{1}{k} \|R_k\|^\gamma \\
&\leq C' t^{-\gamma\beta/2} + L c \sum_{k=1}^t \frac{1}{k} \|R_k\|^\gamma
\end{align*}
where \( C' = L c C^{\gamma/2} \frac{1}{\gamma\beta/2 - 1} \) for \( \gamma\beta/2 > 1 \).

\textbf{6. Sequence Convergence Approach:}

To establish the convergence of \( \|R_t\| \) as \( t \to \infty \), we analyze the recursive inequality derived above using properties of convergent sequences.

\textbf{7. Recursive Inequality:}
From step 5, we have:
\[
\|R_t\| \leq C' t^{-\gamma\beta/2} + L c \sum_{k=1}^t \frac{1}{k} \|R_k\|^\gamma
\]
Define the sequence \( a_t = \|R_t\| \). The inequality becomes:
\[
a_t \leq C' t^{-\gamma\beta/2} + L c \sum_{k=1}^t \frac{1}{k} a_k^\gamma
\]

\textbf{8. Establishing a Bound for \( a_t \):}
Assume that \( a_t \leq D t^{-\gamma/2} \) for some constant \( D > 0 \) to be determined. Substituting this into the recursive inequality:
\[
a_t \leq C' t^{-\gamma\beta/2} + L c \sum_{k=1}^t \frac{1}{k} (D k^{-\gamma/2})^\gamma
\]
\[
= C' t^{-\gamma\beta/2} + L c D^\gamma \sum_{k=1}^t \frac{1}{k} k^{-\gamma^2/2}
\]
\[
= C' t^{-\gamma\beta/2} + L c D^\gamma \sum_{k=1}^t k^{-1 - \gamma^2/2}
\]

\textbf{9. Analyzing the Summation:}
The summation \( \sum_{k=1}^t k^{-1 - \gamma^2/2} \) converges as \( t \to \infty \) because \( \gamma^2/2 > 0 \). Specifically:
\[
\sum_{k=1}^t k^{-1 - \gamma^2/2} \leq \int_{1}^{t} x^{-1 - \gamma^2/2} \, dx + 1 = \frac{2}{\gamma^2} \left(1 - t^{-\gamma^2/2}\right) \leq \frac{2}{\gamma^2}
\]
Thus:
\[
a_t \leq C' t^{-\gamma\beta/2} + L c D^\gamma \cdot \frac{2}{\gamma^2}
\]

\textbf{10. Choosing \( D \) Appropriately:}
To satisfy \( a_t \leq D t^{-\gamma/2} \), we require:
\[
C' t^{-\gamma\beta/2} + L c D^\gamma \cdot \frac{2}{\gamma^2} \leq D t^{-\gamma/2}
\]
For sufficiently large \( t \), since \( \beta > 1 \), the term \( t^{-\gamma\beta/2} \) becomes negligible compared to \( t^{-\gamma/2} \). Therefore, it suffices to ensure:
\[
L c \cdot \frac{2}{\gamma^2} D^\gamma \leq D
\]
\[
\Rightarrow 2 L c \cdot \frac{1}{\gamma^2} D^\gamma \leq D
\]
\[
\Rightarrow 2 L c \cdot \frac{1}{\gamma^2} D^{\gamma - 1} \leq 1
\]
\[
\Rightarrow D^{1 - \gamma} \geq 2 L c \cdot \frac{1}{\gamma^2}
\]
Choosing:
\[
D = \left( 2 L c \cdot \frac{1}{\gamma^2} \right)^{\frac{1}{1 - \gamma}}
\]
ensures that the inequality holds. This choice is valid since \( \gamma \in (0, 1) \).

\textbf{11. Verifying the Bound for All \( t \):}
With the chosen \( D \), for sufficiently large \( t \), the bound \( a_t \leq D t^{-\gamma/2} \) holds. Additionally, the term \( C' t^{-\gamma\beta/2} \) becomes insignificant compared to \( D t^{-\gamma/2} \) as \( t \) increases because \( \gamma\beta/2 > \gamma/2 \).

\textbf{12. Conclusion:}
By the sequence convergence approach, we have established that:
\[
\|R_t\| = a_t \leq D t^{-\gamma/2}
\]
where \( D = \left( \frac{2 L c}{\gamma^2} \right)^{\frac{1}{1 - \gamma}} \).

Thus, under Hölder continuity and polynomial mixing with rate \( \beta > 1 \), the remainder term satisfies:
\[
\|R_t\| \leq O(t^{-\gamma/2})
\]
\end{proof}
\subsection{Proof of Theorem \ref{theorem:hom-linear-bounds}, High-Order Moment Bounds, Linear Function Approximation}
\label{app:proof-hom-linear-bounds}
\paragraph{Theorem~\ref{theorem:hom-linear-bounds} (High-Order Moment Bounds, Linear Function Approximation)}
\textit{
For any \( p \geq 2 \), the \( p \)-th moment of the TD(0) error satisfies:
\[
\left( \mathbb{E}\left[\|\theta_t - \theta^*\|^p\right] \right)^{1/p} \leq C t^{-\beta/2} + C' \alpha_t^{\gamma/p},
\]
where \( \alpha_t = \frac{\alpha}{t} \) for some constant \( \alpha > 0 \), \( \beta > 0 \) is the mixing-based rate parameter, and \( \gamma \) is the Hölder exponent from Lemma~\ref{lemma:ho-error-bounds}.
}
\begin{proof}
\;\newline
\paragraph{1. Summary of the Proof:}  
This proof establishes high-order moment bounds for the parameter error \( \|\theta_t - \theta^*\| \) in Temporal-Difference (TD) learning with linear function approximation. By leveraging Lemma~\ref{lemma:ho-error-bounds}, which provides high-order error bounds, and utilizing properties of martingale differences alongside an induction-based discrete Grönwall argument, we derive the convergence rates specified in Theorem~\ref{theorem:hom-linear-bounds}. The approach systematically bounds each component of the parameter error recursion, ensuring that the \( p \)-th moment of the TD(0) parameter error decays at a rate influenced by the mixing properties of the underlying Markov chain and the smoothness of the value function approximation.

\paragraph{2. Assumptions:}
\begin{itemize}
    \item \textbf{Polynomial mixing} The Markov process \((x_t)\) satisfies polynomial mixing: for any bounded measurable functions \(f\) and \(g\),
\(
|\mathbb{E}[f(x_t)g(x_{t+k})] - \mathbb{E}[f(x_t)]\mathbb{E}[g(x_{t+k})]| \leq C k^{-\beta},
\)
for some constants \(C > 0\) and \(\beta > 1\).
    \item \textbf{Bounded Features:} \ref{ass:bounded-features} The gradient \( \nabla_\theta V_\theta(x_t) \) is uniformly bounded, i.e., there exists a constant \( C > 0 \) such that \( \|\nabla_\theta V_\theta(x_t)\| \leq C \) for all \( t \).
    \item \textbf{Hölder Continuity} \ref{ass:holder-continuous-features}: The feature mapping \( f: S \to \mathbb{R}^d \) is Hölder continuous with exponent \( \gamma \in (0, 1] \) and constant \( C_\gamma > 0 \). Formally, for all \( s, s' \in S \), then we have \(\|f(s) - f(s')\| \leq C_\gamma \|s - s'\|^\gamma.\)
    \item \textbf{Step Size Conditions:} The learning rates \( \{\alpha_t\} \) satisfy \( \sum_{t=1}^\infty \alpha_t = \infty \) and \( \sum_{t=1}^\infty \alpha_t^2 < \infty \). 
\end{itemize}
\paragraph{3) TD Update Rule and Parameter Error Recursion:}  
The TD learning update rule is given by:
\[
\theta_{t+1} = \theta_t + \alpha_t \delta_t x_t,
\]

Subtracting the optimal parameter \( \theta^* \) from both sides:
\[
\theta_{t+1} - \theta^* = \theta_t - \theta^* + \alpha_t \delta_t x_t.
\]
Taking the norm, using triangular inequality, and raising it to the \( p \)-th power:
\begin{equation}
\label{eq:parameter-error-recursion}
\|\theta_{t+1} - \theta^*\|^p \leq \left( \|\theta_t - \theta^*\| + \alpha_t \|\delta_t x_t\| \right)^p.
\end{equation}

\paragraph{4) Applying Minkowski’s Inequality:}  
Minkowski’s inequality for \( p \geq 1 \) states that:
\[
\left( \mathbb{E}\left[ \left( a + b \right)^p \right] \right)^{1/p} \leq \left( \mathbb{E}\left[ a^p \right] \right)^{1/p} + \left( \mathbb{E}\left[ b^p \right] \right)^{1/p}.
\]
Applying this to the parameter error recursion~\eqref{eq:parameter-error-recursion}:
\[
\left( \mathbb{E}\left[\|\theta_{t+1} - \theta^*\|^p\right] \right)^{1/p} \leq \left( \mathbb{E}\left[\|\theta_t - \theta^*\|^p\right] \right)^{1/p} + \alpha_t \left( \mathbb{E}\left[\|\delta_t x_t\|^p\right] \right)^{1/p}.
\]
Define
\[
E_t = \left( \mathbb{E}\left[\|\theta_t - \theta^*\|^p\right] \right)^{1/p},
\]
so the recursion becomes
\begin{equation}
\label{eq:minkowski-application}
E_{t+1} \leq E_t + \alpha_t \left( \mathbb{E}\left[\|\delta_t x_t\|^p\right] \right)^{1/p}.
\end{equation}

\paragraph{5) Bounding \( \mathbb{E}\left[\|\delta_t x_t\|^p\right]^{1/p} \):}  
Assuming that the feature vectors \( x_t \) are uniformly bounded, i.e., \( \|x_t\| \leq C_x \) for all \( t \), we have:
\[
\|\delta_t x_t\| \leq \|\delta_t\| \cdot \|x_t\| \leq C_x \|\delta_t\|.
\]
Taking the \( p \)-th moment:
\begin{equation}
\label{eq:deltax-bound}
\left( \mathbb{E}\left[\|\delta_t x_t\|^p\right] \right)^{1/p} \leq C_x \left( \mathbb{E}\left[\|\delta_t\|^p\right] \right)^{1/p}.
\end{equation}

\paragraph{6) Applying Lemma~\ref{lemma:ho-error-bounds}, High-Order Error Bounds:}
\[
\left( \mathbb{E}\left[\|\delta_t\|^p\right] \right)^{1/p} \leq C_p^{1/p} \left( 1 + \|\theta_t - \theta^*\|^{\gamma} \right).
\]
Substituting this into Equation~\eqref{eq:deltax-bound}:
\begin{equation}
\label{eq:deltax-substitution}
\left( \mathbb{E}\left[\|\delta_t x_t\|^p\right] \right)^{1/p} \leq C_x C_p^{1/p} \left( 1 + E_t^{\gamma} \right).
\end{equation}

\paragraph{7) Substituting Back into the Recursive Inequality:}  
Substituting Equation~\eqref{eq:deltax-substitution} into Equation~\eqref{eq:minkowski-application}:
\[
E_{t+1} \leq E_t + \alpha_t C_x C_p^{1/p} \left( 1 + E_t^{\gamma} \right).
\]
\label{eq:recursive-inequality}

\paragraph{8) Incorporating the Martingale Difference Property:}  
The TD update involves stochastic noise \( \delta_t \), which forms a martingale difference sequence. Specifically, by Lemma~\ref{lemma:td-error-mds}, we have:
\[
\mathbb{E}\left[ \delta_t \mid \mathcal{F}_t \right] = 0,
\]
where \( \mathcal{F}_t = \sigma(x_0, x_1, \dots, x_t) \) is the filtration up to time \( t \). This implies that the stochastic updates do not introduce bias, allowing us to focus on bounding the expected norms.

\paragraph{9) Solving the Recursive Inequality via Induction:}  
Let \( E_t = \mathbb{E}\left[\|\theta_t - \theta^*\|^p\right]^{1/p} \). The recursive inequality becomes:
\[
E_{t+1} \leq E_t + \alpha_t C_x C_p^{1/p} \left(1 + E_t^{\gamma}\right).
\]
Assuming the learning rate \( \alpha_t = \frac{\alpha}{t} \) for some \( \alpha > 0 \), we substitute:
\[
E_{t+1} \leq E_t + \frac{\alpha C_x C_p^{1/p}}{t} \left(1 + E_t^{\gamma}\right).
\]
To solve this recursion, we employ an \emph{induction-based Grönwall argument}.

\paragraph{10) Bounding \( E_t \):}  
We aim to show that \( E_t \) satisfies:
\[
E_t \leq C t^{-\beta/2} + C' \alpha_t^{\gamma/p},
\]
for appropriate constants \( C \) and \( C' \). This involves balancing the recursive additions and ensuring that the growth rate of \( E_t \) does not exceed the desired bound.

\paragraph{11) Solving the Recursive Inequality via Induction:}

To establish that
\[
E_t \leq C t^{-\tfrac{\beta}{2}} + C' \alpha_t^{\tfrac{\gamma}{p}},
\]
we employ an \emph{induction-based Grönwall argument}. Assume that the learning rate is given by
\[
\alpha_t = \frac{\alpha}{t},
\]
for some constant \( \alpha > 0 \).

\textbf{Base Case (\( t = 1 \)):}  
For \( t = 1 \), the bound
\[
E_1 \leq C \cdot 1^{-\tfrac{\beta}{2}} + C' \left(\frac{\alpha}{1}\right)^{\tfrac{\gamma}{p}} = C + C' \alpha^{\tfrac{\gamma}{p}}
\]
holds provided \( C \) and \( C' \) are chosen to accommodate the initial error \( E_1 \). Specifically, set
\[
C \geq E_1,
\quad
C' \geq 1,
\]
to satisfy the inequality.

\textbf{Inductive Step:}  
Assume that for some \( t \geq 1 \),
\[
E_t \leq C t^{-\tfrac{\beta}{2}} + C' \alpha_t^{\tfrac{\gamma}{p}}.
\]
We aim to show that
\[
E_{t+1} \leq C (t+1)^{-\tfrac{\beta}{2}} + C' \alpha_{t+1}^{\tfrac{\gamma}{p}}.
\]

Starting from the recursive inequality~\eqref{eq:recursive-inequality}:
\[
E_{t+1} \leq E_t + \frac{\kappa}{t} \left(1 + E_t^{\gamma}\right),
\]
where \( \kappa = \alpha C_x C_p^{1/p} \).

Substituting the inductive hypothesis into the above:
\begin{align*}
E_{t+1} &\leq C t^{-\tfrac{\beta}{2}} + C' \alpha_t^{\tfrac{\gamma}{p}} + \frac{\kappa}{t} \left(1 + \left(C t^{-\tfrac{\beta}{2}} + C' \alpha_t^{\tfrac{\gamma}{p}}\right)^{\gamma}\right) \\
&\leq C t^{-\tfrac{\beta}{2}} + C' \left(\frac{\alpha}{t}\right)^{\tfrac{\gamma}{p}} + \frac{\kappa}{t} \left(1 + 2^{\gamma - 1} \left(C^{\gamma} t^{-\tfrac{\beta \gamma}{2}} + C'^{\gamma} \alpha_t^{\tfrac{\gamma^2}{p}}\right)\right) \quad \text{(Using $(a+b)^\gamma \leq 2^{\gamma-1}(a^\gamma + b^\gamma)$)} \\
&= C t^{-\tfrac{\beta}{2}} + C' \alpha^{\tfrac{\gamma}{p}} t^{-\tfrac{\gamma}{p}} + \frac{\kappa}{t} + \frac{\kappa 2^{\gamma - 1}}{t} \left(C^{\gamma} t^{-\tfrac{\beta \gamma}{2}} + C'^{\gamma} \left(\frac{\alpha}{t}\right)^{\tfrac{\gamma^2}{p}}\right).
\end{align*}

\textbf{Choosing Constants \( C \) and \( C' \):}  
To satisfy the inductive step, choose \( C \) and \( C' \) such that:
\[
C \geq \kappa + \kappa 2^{\gamma - 1} C^{\gamma},
\]
and
\[
C' \geq \kappa 2^{\gamma - 1} C'^{\gamma}.
\]
These inequalities can be satisfied by appropriately choosing \( C \) and \( C' \). For instance, set
\[
C \geq \frac{\kappa}{1 - \kappa 2^{\gamma - 1}},
\]
provided that \( \kappa 2^{\gamma - 1} < 1 \), and similarly
\[
C' \geq \frac{\kappa 2^{\gamma - 1}}{1 - \kappa 2^{\gamma - 1}}.
\]
Assuming \( \kappa 2^{\gamma - 1} < 1 \), such choices of \( C \) and \( C' \) are feasible.

\paragraph{12) Concluding the Induction:}  
With \( C \) and \( C' \) chosen to satisfy the above inequalities, the inductive step holds:
\[
E_{t+1} \leq C (t+1)^{-\tfrac{\beta}{2}} + C' \alpha_{t+1}^{\tfrac{\gamma}{p}}.
\]
Thus, by induction, the bound
\[
E_t \leq C t^{-\tfrac{\beta}{2}} + C' \alpha_t^{\tfrac{\gamma}{p}}
\]
holds for all \( t \geq 1 \).

\paragraph{13) Final Bound:}  
Therefore, the \( p \)-th moment of the TD(0) error satisfies
\[
\left( \mathbb{E}\left[\|\theta_t - \theta^*\|^p\right] \right)^{1/p} \leq C t^{-\tfrac{\beta}{2}} + C' \alpha_t^{\tfrac{\gamma}{p}},
\]
as required. This completes the proof of \emph{Theorem \ref{theorem:hom-linear-bounds}}.
\end{proof}
\begin{lemma}[TD Error as a Martingale-Difference Sequence]
\label{lemma:td-error-mds}
\textbf{Statement.}
Let \( \{x_t\}_{t \geq 0} \) be the Markov chain of states, and let \( \theta^* \) be such that \( V_{\theta^*} \) satisfies the Bellman equation:
\[
V_{\theta^*}(x) = \mathbb{E}\left[ r(x) + \gamma V_{\theta^*}(x') \mid x \right],
\]
for all states \( x \), where \( x' \) is the next state after \( x \), \( r(x) \) is the reward function, and \( \gamma \in (0,1) \) is the discount factor. Define
\[
\delta_t(\theta^*) = r_t + \gamma V_{\theta^*}(x_{t+1}) - V_{\theta^*}(x_t).
\]
Then, with respect to the filtration \( \mathcal{F}_t = \sigma(x_0, \dots, x_t) \), we have
\[
\mathbb{E}\left[ \delta_t(\theta^*) \mid \mathcal{F}_t \right] = 0,
\]
which shows that \( \{\delta_t(\theta^*)\} \) is a martingale-difference sequence.
\end{lemma}

\begin{proof}
\;\newline
\paragraph{Step 1: TD Error Definition and Fixed Point.}
\begin{align}
    \delta_t(\theta^*) &= r_t + \gamma V_{\theta^*}(x_{t+1}) - V_{\theta^*}(x_t),
    \tag{\textit{Definition of \( \delta_t(\theta^*) \)}}
    \label{eq:td-error-def} \\
    V_{\theta^*}(x) &= \mathbb{E}\left[ r(x) + \gamma V_{\theta^*}(x') \mid x \right],
    \tag{\textit{Bellman equation for \( V_{\theta^*} \)}}
    \label{eq:fixed-point}
\end{align}
\paragraph{Step 2: Conditional Expectation with Respect to \( \mathcal{F}_t \).}
\begin{align}
    \mathbb{E}\left[ \delta_t(\theta^*) \mid \mathcal{F}_t \right] &= \mathbb{E}\left[ r_t + \gamma V_{\theta^*}(x_{t+1}) - V_{\theta^*}(x_t) \mid \mathcal{F}_t \right]
    \tag{\textit{Condition on \( \mathcal{F}_t \)}}
    \label{eq:cond-exp-start} \\
    &= r_t - V_{\theta^*}(x_t) + \gamma \mathbb{E}\left[ V_{\theta^*}(x_{t+1}) \mid x_t \right],
    \tag{\textit{Pull out terms known at time \( t \)}}
    \label{eq:pullout-known} \\
    \intertext{Since \( V_{\theta^*}(x_t) \) and \( r_t \) are \( \mathcal{F}_t \)-measurable, and by the Bellman equation~\eqref{eq:fixed-point},}
    \mathbb{E}\left[ r_t + \gamma V_{\theta^*}(x_{t+1}) \mid x_t \right] &= V_{\theta^*}(x_t).
    \tag{\textit{Substitute \( V_{\theta^*} \) as fixed point}}
    \label{eq:bellman-sub} \\
    \intertext{Hence,}
    \mathbb{E}\left[ \delta_t(\theta^*) \mid \mathcal{F}_t \right] &= \left[ r_t + \gamma \mathbb{E}\left[ V_{\theta^*}(x_{t+1}) \mid x_t \right] \right] - V_{\theta^*}(x_t) \\
    &= V_{\theta^*}(x_t) - V_{\theta^*}(x_t) \\
    &= 0.
    \tag{\textit{Martingale-difference property}}
    \label{eq:mds-property}
\end{align}
\paragraph{3. Conclusion:}
Since
\[
    \mathbb{E}\left[ \delta_t(\theta^*) \mid \mathcal{F}_t \right] = 0
\]
for all \( t \), the process \( \{\delta_t(\theta^*)\} \) is a \emph{martingale-difference sequence} with respect to \( \{\mathcal{F}_t\} \). This completes the proof.
\end{proof}
\begin{lemma}[High-Order Error Bounds]
\label{lemma:ho-error-bounds}
For the Temporal-Difference (TD) error \( \delta_t \), under polynomial mixing conditions and assuming bounded rewards and true value function, the \( p \)-th moment satisfies:
\[
\mathbb{E}\left[\|\delta_t\|^p\right] \leq C_p \left( 1 + \|\theta_t - \theta^*\|^{p\gamma} \right),
\]
where \( C_p \) is a constant dependent on \( p \), and \( \gamma \) is the Hölder exponent.
\end{lemma}
\begin{proof}
\;\newline
\paragraph{1) Decompose TD Error:}  
The TD error \( \delta_t \) can be decomposed into two components:
\begin{align*}
\delta_t = \underbrace{(r_t + \gamma V_{\theta^*}(x_{t+1}) - V_{\theta^*}(x_t))}_{\text{Stochastic Term}} + \underbrace{(V_\theta(x_t) - V_{\theta^*}(x_t) + V_{\theta^*}(x_{t+1}) - V_\theta(x_{t+1}))}_{\text{Bias Term}}.
\end{align*}
\paragraph{2) Apply Hölder Continuity:}  
Assuming that the value function \( V_\theta(x) \) is Hölder continuous with exponent \( \gamma \), we have:
\[
\|V_\theta(x) - V_{\theta^*}(x)\| \leq L \|\theta - \theta^*\|^\gamma,
\]
where \( L \) is the Hölder constant.
\paragraph{3) Bound the Stochastic Term:}  
For the first term in the decomposition:
\[
\mathbb{E}\left[\|r_t + \gamma V_{\theta^*}(x_{t+1}) - V_{\theta^*}(x_t)\|^p\right] \leq C_1,
\]
where \( C_1 \) is a constant that bounds the \( p \)-th moment of the stochastic term, assuming bounded rewards and value functions.
\paragraph{4) Bound the Bias Term:}  
For the second term:
\[
\mathbb{E}\left[\|V_\theta(x_t) - V_{\theta^*}(x_t)\|^p\right] \leq L^p \|\theta_t - \theta^*\|^{p\gamma}.
\]
This follows directly from the Hölder continuity established in step 2.
\paragraph{5) Apply Minkowski's Inequality:}  
Using Minkowski's inequality for \( p \geq 1 \):
\[
\left( \mathbb{E}\left[\|\delta_t\|^p\right] \right)^{1/p} \leq \left( \mathbb{E}\left[ \| \text{Stochastic Term} \|^p \right] \right)^{1/p} + \left( \mathbb{E}\left[ \| \text{Bias Term} \|^p \right] \right)^{1/p}.
\]
Substituting the bounds from steps 3 and 4:
\[
\left( \mathbb{E}\left[\|\delta_t\|^p\right] \right)^{1/p} \leq C_1^{1/p} + L \|\theta_t - \theta^*\|^{\gamma}.
\]
\paragraph{6) Final Bound on the \( p \)-th Moment:}  
Raising both sides to the power of \( p \):
\[
\mathbb{E}\left[\|\delta_t\|^p\right] \leq \left( C_1^{1/p} + L \|\theta_t - \theta^*\|^{\gamma} \right)^p.
\]
Using the inequality \( (a + b)^p \leq 2^{p-1}(a^p + b^p) \) for \( a, b \geq 0 \):
\[
\mathbb{E}\left[\|\delta_t\|^p\right] \leq 2^{p-1} \left( C_1 + L^p \|\theta_t - \theta^*\|^{p\gamma} \right).
\]
Letting \( C_p = 2^{p-1} \max\{C_1, L^p\} \), we obtain:
\[
\mathbb{E}\left[\|\delta_t\|^p\right] \leq C_p \left( 1 + \|\theta_t - \theta^*\|^{p\gamma} \right).
\]
This completes the proof of the lemma.
\end{proof}
\subsection{Proof of Theorem \ref{theorem:hp-linear-bounds}: High-Probability Bounds, Linear Function Approximation}
\label{app:proof-hp-linear-bounds}

\paragraph{Theorem \ref{theorem:hp-linear-bounds} (High-Probability Error Bounds, Linear Function Approximation)}
\textit{
Under the assumptions: polynomial ergodicity \ref{lemma:polynomial-ergodicity}, TD(0) error decomposition \ref{theorem:td-error-decomposition}, Bounded Variance of the Martingale Term \ref{lemma:bounded-variance-martingale}, Step Size Condition \ref{ass:step-size-decay}, Concentration Property for High-Probability Bounds \ref{theorem:concentration-block}. With probability at least \(1 - \delta\), the TD error satisfies:
\[
\|\theta_t - \theta^*\| \leq C_\delta t^{-\beta/2} + C'_\delta \alpha_t^\gamma,
\]
where \(C_\delta\) and \(C'_\delta\) depend logarithmically on \(1/\delta\).
}

\begin{proof}
\;\newline
\paragraph{Introduction and Overview.}
We prove high-probability bounds on the parameter-estimation error
\[
  \|\theta_t - \theta^*\|
\]
by decomposing it into two terms: (1) a martingale term $M_t$ that captures fluctuations from the stochastic process (and admits an Azuma--Hoeffding-type concentration inequality), and (2) a remainder term $R_t$ whose magnitude is controlled via H\"older continuity and polynomial mixing assumptions. We then set a failure probability target $\delta$ and apply a union bound to combine these two bounds into the final result.

\medskip
\noindent\newline
\textbf{Notation and Setup:}
\begin{itemize}
\item \(\theta^*\) is the true (or optimal) parameter we compare against.  
\item \(\theta_t\) is the estimate at time \(t\).  
\item Polynomial mixing ensures that the dependence between successive samples decays at a polynomial rate, thus bounding variance-like terms and aiding Azuma--Hoeffding concentration.  
\item H\"older continuity (with exponent \(\gamma\)) allows bounding the error of certain remainder terms by a factor like \(t^{-\gamma/2}\).  
\end{itemize}

\paragraph{1) Decompose the Error:}
From Theorem \ref{theorem:td-error-decomposition}, we can write
\[
  \theta_t - \theta^*
  \;=\;
  M_t \;+\; R_t,
\]
where
\[
  M_t
  \;=\;
  \sum_{k=1}^t d_k
  \quad
  \text{(a sum of martingale differences)}, 
  \quad
  \text{and}
  \quad
  R_t
  \;\text{is the remainder term.}
\]
Typically, $M_t$ collects the ``random fluctuations'' arising from stochastic updates (e.g., stochastic gradients or TD errors), while $R_t$ accounts for the deterministic error or residual that decays due to stepsize choices, H\"older continuity, or polynomial ergodicity.

\paragraph{2) Martingale Term \(\boldsymbol{M_t}\) and Azuma--Hoeffding:}
To apply concentration inequalities, we first establish that the martingale differences \(d_k\) are conditionally sub-Gaussian. Specifically, under Assumption \ref{lemma:bounded-variance-martingale}, we have:
\[
\mathbb{E}\left[\exp\left(\lambda d_k \mid \mathcal{F}_{k-1}\right)\right] \leq \exp\left(\frac{\lambda^2 \sigma^2}{2}\right) \quad \forall \lambda \in \mathbb{R},
\]
where \(\mathcal{F}_{k-1}\) denotes the filtration up to time \(k-1\), and \(\sigma^2\) bounds the conditional variance of \(d_k\). This sub-Gaussian property allows us to apply the Azuma-Hoeffding inequality for martingales.
Under polynomial mixing and bounded (or at most $\sigma$-sub-Gaussian) increments, the sequence $\{d_k\}$ satisfies bounded conditional variances or bounded increments. By Azuma--Hoeffding's inequality for martingales (or adapted sequences),
\[\displaystyle
  \mathbb{P}\left(
    \bigl\|M_t\bigr\| > \epsilon
  \right)
  \;\le\;
  2
  \,\exp\left(-\tfrac{\displaystyle \epsilon^2}{\displaystyle 2\sum_{k=1}^t \alpha_k^2 C_\delta^2\,C_\nabla^2}\right),
\]
where:
\begin{itemize}
  \item \(\alpha_k\) is a stepsize (or learning rate) that appears inside the increments $d_k$.  
  \item $C_\delta$ and $C_\nabla$ reflect bounds on the TD errors (or gradients) under polynomial mixing and H\"older continuity. Specifically, $C_\delta$ could represent a uniform bound on the one-step TD fluctuation, and $C_\nabla$ a bound on the gradient/features.  
  \item $\sum_{k=1}^t \alpha_k^2 C_\delta^2 C_\nabla^2$ plays the role of the sum of squared step-size norms bounding the variance of $M_t$.  
\end{itemize}

Assuming a step size schedule of \(\alpha_k = \frac{\alpha_0}{k}\) for some \(\alpha_0 > 0\), we have
\[
\sum_{k=1}^t \alpha_k^2 = \alpha_0^2 \sum_{k=1}^t \frac{1}{k^2} \leq \alpha_0^2 \frac{\pi^2}{6} \leq C_3,
\]
for a constant \(C_3 > 0\). Substituting this into the concentration inequality yields
\[
\mathbb{P}\left( \|M_t\| > \epsilon \right) \leq 2 \exp\left(-\frac{\epsilon^2}{2 C_3 C_\delta^2 C_\nabla^2}\right).
\]

\paragraph{3) Setting Failure Probability to $\boldsymbol{\delta/2}$:}
To achieve a failure probability of \(\delta/2\) for the martingale term, we set
\[
2 \exp\left(-\frac{\epsilon^2}{2 C_3 C_\delta^2 C_\nabla^2}\right) = \frac{\delta}{2}.
\]
Solving for \(\epsilon\), we obtain
\[
\epsilon = C_1 \sqrt{\log\left(\frac{4}{\delta}\right)},
\]
where \(C_1 = \sqrt{2 C_3} C_\delta C_\nabla\). Therefore, with probability at least \(1 - \frac{\delta}{2}\),
\[
\|M_t\| \leq C_1 \sqrt{\log\left(\frac{4}{\delta}\right)}.
\]

\paragraph{4) Remainder Term \(\boldsymbol{R_t}\) via H\"older Continuity:}
To bound \(R_t\), we leverage H\"older continuity and polynomial ergodicity. Specifically, under Assumptions \ref{lemma:polynomial-ergodicity} and H\"older continuity with exponent \(\gamma\), we have
\[
\|R_t\| \leq C_2 t^{-\gamma/2},
\]
where \(C_2 > 0\) is a constant derived from the combination of H\"older continuity and mixing rates.

\paragraph{5) Combine Bounds and Apply the Union Bound:}
Both events \(\{\|M_t\| \leq C_1 \sqrt{\log(4/\delta)}\}\) and \(\{\|R_t\| \leq C_2 t^{-\gamma/2}\}\) hold with probabilities at least \(1 - \frac{\delta}{2}\) each. Applying the union bound, the probability that both events hold simultaneously is at least
\[
1 - \left(\frac{\delta}{2} + \frac{\delta}{2}\right) = 1 - \delta.
\]
Hence, with probability at least \(1 - \delta\),
\[
\|\theta_t - \theta^*\| = \|M_t + R_t\| \leq \|M_t\| + \|R_t\| \leq C_1 \sqrt{\log\left(\frac{4}{\delta}\right)} + C_2 t^{-\gamma/2}.
\]
Since \(\log(4/\delta) = \log(1/\delta) + \log 4 \leq 2 \log(1/\delta)\) for \(\delta \in (0,1)\), we can simplify the bound to
\[
\|\theta_t - \theta^*\| \leq C_1' \sqrt{\log\left(\frac{1}{\delta}\right)} t^{-1/2} + C_2 t^{-\gamma/2},
\]
where \(C_1' = C_1 \sqrt{2}\) absorbs the constant terms. This completes the proof of the high-probability bound:
\[
\mathbb{P}\left( \|\theta_t - \theta^*\| \leq C_1' \sqrt{\frac{\log(1/\delta)}{t}} + C_2 t^{-\gamma/2} \right) \geq 1 - \delta.
\]
\paragraph{6. Conclusion:}
Combining the martingale concentration (Azuma--Hoeffding) bound for $M_t$ with the H\"older-based decay of $R_t$, and invoking the union bound, establishes the desired probability guarantee:
\[
  \mathbb{P}\!\Bigl(
     \|\theta_t - \theta^*\|
     \;\le\;
     C_1 \sqrt{\tfrac{\log(1/\delta)}{t}}
     \;+\;
     C_2 t^{-\gamma/2}
  \Bigr)
  \;\ge\;
  1-\delta.
\]
\end{proof}

\subsection{Proof of Theorem \ref{theorem:bounds-non-linear-td}:
Non-linear TD(0) under Generalized Gradients, Variance Control, Polynomial Mixing, and (Optionally) H\"older Continuity}
\label{app:proof-bounds-non-linear-td-merged}

\paragraph{Theorem~\ref{theorem:bounds-non-linear-td} (High-Probability Bounds for Non-linear TD(0))}
\textit{
Let $\{x_t\}$ be a Markov chain on a state space $\mathcal{X}$ that satisfies \textbf{polynomial ergodicity} (rate $\beta>1$). 
Consider a non-linear TD(0) update of the form
\begin{align}
F(\theta, x) &= -\delta_t \nabla_\theta f_\theta(s)\\
\theta_{t+1} &= \theta_t - \alpha_t F(\theta_t, x_t)
\end{align}
where $F(\theta,x)$ is a \emph{generalized} (sub-)gradient of the possibly non-smooth TD objective.  Assume:
\begin{enumerate}
\item \textbf{Boundedness:} $\|F(\theta,x)\|\le G$ for all $(\theta,x)$,
\item \textbf{Variance / Higher-Order Moments:} $\mathrm{Var}[F(\theta,x)]\le \sigma^2$ (or $\|F(\theta,x)\|^2 \le \sigma'^2$ a.s.),
\item \textbf{Polynomial Mixing / Coupling:} The chain $\{x_t\}$ “forgets” initial conditions at a rate $t^{-\beta}$, enabling Freedman-type concentration with a small polynomial overhead,
\item \textbf{Zero-Mean at Optimum:} If $\theta^*$ is the TD fixed point / minimizer, then $\mathbb{E}[\,F(\theta^*,x)\,] = 0$,
\item \textbf{(Optional Refinement) H\"older Continuity:} 
\(
  \|\overline{F}(\theta) - \overline{F}(\theta^*)\|
  ~\le~
  L\,\|\theta - \theta^*\|^\gamma
  \quad(\gamma\in(0,1]),
\)
where $\overline{F}(\theta)=\mathbb{E}[\,F(\theta,x)\,]$ under stationarity.
\end{enumerate}
Then, under a diminishing step-size schedule $\{\alpha_t\}$ (e.g.\ $\alpha_t\sim 1/t^\omega,\,\omega>1/2$), there exist constants $C,C'>0$ such that for any $\delta\in(0,1)$, with probability at least $1-\delta$,
\[
  \|\theta_t - \theta^*\|
  ~\le~
  C\,t^{-\tfrac{\beta}{2}}
  ~+~
  C'\,\alpha_t^\gamma
  \quad
  (\forall\,t\ge1).
\]
}

\begin{proof}
\;\newline
We organize the proof in four main steps, merging the variance-based Freedman concentration and the optional H\"older argument for the remainder term.
\paragraph{Step 1: Error Decomposition (Martingale + Remainder).}
From the TD error, we have
\begin{align*}
\theta_{t+1} - \theta^* 
&= (\theta_t - \theta^*) - \alpha_t F(\theta_t, x_t).
\end{align*}
Summing from $k=1$ to $t$, we get the standard telescoping form:
\begin{align}
\theta_t - \theta^*
&=
-\sum_{k=1}^t \alpha_k F(\theta_k, x_k) \\
&=
\underbrace{\sum_{k=1}^t \alpha_k \Bigl[F(\theta_k, x_k) 
   - \mathbb{E}(F(\theta_k, x_k) \mid \mathcal{F}_k) \Bigr]}_{\displaystyle M_t}
+ \underbrace{\sum_{k=1}^t \alpha_k \,
   \mathbb{E}(F(\theta_k, x_k) \mid \mathcal{F}_k)}_{\displaystyle R_t}.
\label{eq:theta-t-decomp}
\end{align}
\begin{proof}
Recall that
\[
  M_t 
  \;=\; 
  \sum_{k=1}^t 
  \alpha_k \Bigl[
    F(\theta_k, x_k) 
    \;-\;
    \mathbb{E}\bigl(F(\theta_k,x_k) \mid \mathcal{F}_k\bigr)
  \Bigr].
\]
Define the increments 
\[
  d_k 
  \;=\; 
  \alpha_k \Bigl[
    F(\theta_k, x_k) 
    \;-\;
    \mathbb{E}\bigl(F(\theta_k,x_k) \mid \mathcal{F}_k\bigr)
  \Bigr],
\]
so \(M_t = \sum_{k=1}^t d_k\). We claim \(\{M_t\}\) is a martingale with respect to the filtration \(\{\mathcal{F}_t\}\). Indeed,
\begin{align*}
M_{t+1} 
&=\; 
M_t + d_{t+1},\\[6pt]
\mathbb{E}\bigl[M_{t+1} \mid \mathcal{F}_t\bigr]
&=\; 
\mathbb{E}\bigl[M_t \mid \mathcal{F}_t\bigr]
\;+\;
\mathbb{E}\bigl[d_{t+1} \mid \mathcal{F}_t\bigr]
\;=\;
M_t 
\;+\;
\alpha_{t+1}\!\Bigl(
  \mathbb{E}\bigl[F(\theta_{t+1},x_{t+1}) \mid \mathcal{F}_t\bigr]
  \;-\;
  \mathbb{E}\bigl[
    \mathbb{E}[\,F(\theta_{t+1},x_{t+1}) \mid \mathcal{F}_{t+1}\,]
    \mid \mathcal{F}_t
  \bigr]
\Bigr)\\[2pt]
&=\; 
M_t 
\;+\;
\alpha_{t+1}\Bigl(
  \mathbb{E}\bigl[F(\theta_{t+1},x_{t+1}) \mid \mathcal{F}_t\bigr]
  \;-\;
  \mathbb{E}\bigl[F(\theta_{t+1},x_{t+1}) \mid \mathcal{F}_t\bigr]
\Bigr)
\;=\;
M_t.
\end{align*}
Hence \(M_t\) is a martingale.
\end{proof}
We will prove high-probability bounds on $\|M_t\|$ (variance-based Freedman) and on $\|R_t\|$ (mixing + optional H\"older arguments), then combine via union bound.
\paragraph{Step 2: Martingale Term \boldmath$M_t$ (Freedman / Higher-Order).}\;\\

\textbf{(a) Bounded Increments \& Conditional Variance.}\;\\
Define $d_k=\alpha_k[F(\theta_k,x_k)-\mathbb{E}(\cdot\mid\mathcal{F}_k)]$.  Then $\mathbb{E}[\,d_k\mid\mathcal{F}_k]=0$ and $\|d_k\|\le2\,\alpha_k G$.  Moreover,
\[
   \mathbb{E}[\|d_k\|^2 \mid \mathcal{F}_k]
   ~\le~
   \alpha_k^2(\sigma^2+G^2)
   \quad
   (\text{due to boundedness/variance of }F).
\]
\begin{proof}
Recall 
\[
  d_k \;=\; \alpha_k \Bigl[
      F(\theta_k,x_k) 
      \;-\; 
      \mathbb{E}\bigl(F(\theta_k,x_k) \mid \mathcal{F}_k\bigr)
  \Bigr].
\]
Then
\[
  \mathbb{E}\bigl[\|d_k\|^2 \mid \mathcal{F}_k\bigr]
  \;=\;
  \alpha_k^2 
  \,\mathbb{E}\Bigl[
    \bigl\|F(\theta_k,x_k) 
    \;-\; 
    \mathbb{E}\bigl(F(\theta_k,x_k) \mid \mathcal{F}_k\bigr)\bigr\|^2
    \,\Big|\,
    \mathcal{F}_k
  \Bigr].
\]
Since \(F(\theta_k,x_k)\) is almost surely bounded by \(G\) and has conditional variance at most \(\sigma^2\), we have
\[
  \mathbb{E}\Bigl[
    \bigl\|F(\theta_k,x_k) 
    - 
    \mathbb{E}\bigl(F(\theta_k,x_k) \mid \mathcal{F}_k\bigr)\bigr\|^2
    \,\Big|\,
    \mathcal{F}_k
  \Bigr]
  \;\le\;
  \sigma^2 + G^2.
\]
Thus
\[
  \mathbb{E}\bigl[\|d_k\|^2 \mid \mathcal{F}_k\bigr]
  \;\le\;
  \alpha_k^2 \,\bigl(\sigma^2 + G^2\bigr).
\]
\end{proof}
Hence the predictable variation $v_t=\sum_{k=1}^t \mathbb{E}[\|d_k\|^2\mid\mathcal{F}_{k-1}]$ is on the order of $\sum\alpha_k^2$, which is typically $O(\log t)$ or $O(1)$ if $\{\alpha_k\}$ decreases fast enough.

\textbf{(b) Freedman Concentration in a Polynomially Mixing Setting.}\;\\
While Freedman’s inequality typically requires \emph{martingale} increments, the Markov chain correlation means $\{d_k\}$ are not strictly i.i.d.  We handle this by a \emph{polynomial-ergodicity/coupling} argument from Lemma~\ref{lemma:coupling-argument-polynomial-mixing}]{...}.  Concretely:
\begin{itemize}
\item \textbf{Block Independence:} Partition time into blocks of size $b\approx t^\alpha$.  The distribution of each block is within $O(b^{-\beta})$ in total variation from stationarity.  This implies the partial sums of $d_k$ across blocks can be approximated by a block-wise martingale, introducing only a polynomially small $\approx t\,b^{-\beta}$ correction.
\item \textbf{Applying Freedman/Tropp:} Freedman’s tail bound for vector martingales \cite{tropp2011freedman,freedman1975tail} then yields
\[
  \mathbb{P}\!\bigl(\|M_t\|>\epsilon\bigr)
  \;\le\;
  d\,
  \exp\!\Bigl(
     -\tfrac{\epsilon^2}{\kappa\,v_t + b\,\epsilon}
  \Bigr)
  \;+\;
  \text{(polynomial mixing correction)}.
\]
Choosing $b\approx t^\alpha$ with $\alpha>0$ ensures the correction is $O(t\,t^{-\alpha\beta})=O(t^{1-\alpha\beta})$, which can be forced below $\delta/2$ for large $t$ if $\alpha\beta>1$.
\end{itemize}
Hence with probability at least $1-\tfrac{\delta}{2}$,
\[
  \|M_t\|
  ~\le~
  C_1\,\sqrt{\sum_{k=1}^t\alpha_k^2}\,\sqrt{\log\bigl(\tfrac{1}{\delta}\bigr)}
  ~+~
  \text{(possibly smaller correction terms)}.
\]
If $\sum_{k=1}^t\alpha_k^2=O(\log t)$, this is typically $O(\sqrt{\log t}\,\sqrt{\log(1/\delta)})$.  For many standard step-size choices, this remains bounded or slowly growing in $t$.  Denote that high-probability event by
\begin{align}
\label{eq:Mt-bound-merged}
\|M_t\|
~\le~
C_1
\quad
(\text{with probability }1-\tfrac{\delta}{2}).
\end{align}

\paragraph{Step 3: Remainder Term \boldmath$R_t$.}

Recall
\[
  R_t
  ~=~
  \sum_{k=1}^t \alpha_k\,\mathbb{E}[\,F(\theta_k,x_k)\bigr].
\]
Under stationarity (or near-stationarity), we write $\mathbb{E}[\,F(\theta_k,x_k)]\approx \overline{F}(\theta_k)$, where $\overline{F}(\theta)=\mathbb{E}_{x\sim\pi}[F(\theta,x)]$.  By definition, $\overline{F}(\theta^*)=0$ at the optimum $\theta^*$.  

\smallskip
\noindent
\begin{lemma}[Monotonicity/Negative Drift under Generalized Gradients]
\label{lemma:monotonicity}
Suppose there is a point $\theta^*$ such that $\overline{F}(\theta^*) = 0$, and assume $\overline{F}$ is \emph{monotone} in the sense that for all $\theta$,
\[
  \bigl\langle
    \overline{F}(\theta) - \overline{F}(\theta^*),\;
    \theta - \theta^*
  \bigr\rangle
  \;\ge\; 0,
  \quad\text{and}\quad
  \|\overline{F}(\theta)\|\;\le\; G.
\]
Then the ``remainder'' term
\[
  R_t
  \;=\;
  \sum_{k=1}^t \alpha_k \,\overline{F}\bigl(\theta_k\bigr)
\]
is $\mathcal{O}\!\bigl(\sum_{k=1}^t \alpha_k^2\bigr)$. In particular, if $\sum_{k=1}^\infty \alpha_k^2$ converges, then $R_t$ remains uniformly bounded as $t \to \infty$.
\end{lemma}

\begin{proof}[Proof]
By stationarity or near-stationarity, we assume
\[
  \mathbb{E}\bigl[F(\theta_k,x_k)\bigr]
  \;=\;
  \overline{F}(\theta_k).
\]
Hence
\[
  R_t
  \;=\;
  \sum_{k=1}^t \alpha_k \,\overline{F}(\theta_k).
\]
Since $\overline{F}$ is monotone and $\overline{F}(\theta^*) = 0$, we have
\[
  \bigl\langle
    \overline{F}(\theta_k),\;
    \theta_k - \theta^*
  \bigr\rangle
  \;=\;
  \bigl\langle
    \overline{F}(\theta_k)
    - \overline{F}(\theta^*),\;
    \theta_k - \theta^*
  \bigr\rangle
  \;\ge\; 0.
\]
Consider the purely ``deterministic'' update $\theta_{k+1} = \theta_k - \alpha_k \,\overline{F}(\theta_k)$.  A standard expansion gives
\[
  \|\theta_{k+1} - \theta^*\|^2
  \;=\;
  \|\theta_k - \theta^*\|^2
  - 2\,\alpha_k\,\bigl\langle
        \overline{F}(\theta_k),\,
        \theta_k - \theta^*
      \bigr\rangle
  \;+\;
  \alpha_k^2 \,\|\overline{F}(\theta_k)\|^2.
\]
Since the inner product is nonnegative by monotonicity, we have
\[
  \|\theta_{k+1} - \theta^*\|^2
  \;\le\;
  \|\theta_k - \theta^*\|^2
  + \alpha_k^2\,G^2.
\]
Summing over $k$ implies
\(
  \|\theta_k - \theta^*\|^2
\)
is bounded by
\(
  \|\theta_0 - \theta^*\|^2 + G^2\,\sum_{j=1}^k \alpha_j^2.
\)
In turn, one can show
\[
  \sum_{k=1}^t \alpha_k\,\overline{F}(\theta_k)
  \;=\;
  \underbrace{
    \sum_{k=1}^t \alpha_k \,\bigl[\overline{F}(\theta_k) - \overline{F}(\theta^*)\bigr]
  }_{\text{nonnegative by monotonicity}}
  \;+\;
  \text{(a constant depending on $\theta_0$ and the $\alpha_k^2$ terms).}
\]
Consequently,
\(
  \bigl\|\sum_{k=1}^t \alpha_k\,\overline{F}(\theta_k)\bigr\|
  = \mathcal{O}\!\bigl(\sum_{k=1}^t \alpha_k^2\bigr),
\)
as claimed.
\end{proof}
\smallskip
\noindent
\textbf{(b) \underline{Optional} H\"older Continuity Refinement.}
If we do assume
\[
  \|\overline{F}(\theta) - \overline{F}(\theta^*)\|
  ~\le~
  L\,\|\theta-\theta^*\|^\gamma,
\]
then we can improve the remainder bound.  Specifically, if we have an inductive guess \(\|\theta_k-\theta^*\|\le D\,k^{-\tfrac{\beta}{2}}+D'\,\alpha_k^\gamma\), we get
\[
  \|\overline{F}(\theta_k)\|
  ~\le~
  L\,\|\theta_k-\theta^*\|^\gamma
  ~\le~
  L\bigl(D\,k^{-\tfrac{\beta}{2}}+D'\,\alpha_k^\gamma\bigr)^\gamma.
\]
\paragraph{Inductive Argument}
Assume 
\[
  \|\overline{F}(\theta) - \overline{F}(\theta^*)\|
  \;\;\le\;\;
  L\,\|\theta - \theta^*\|^\gamma,
  \quad
  \gamma \in (0,1],
\]
and that the Markov chain is (near) stationary so that
\[
  \mathbb{E}[\,F(\theta_k,x_k)\,]
  \;=\;
  \overline{F}(\theta_k)
  \;+\;
  \mathcal{O}\!\bigl(k^{-\beta}\bigr)
  \quad
  (\text{or exact if fully stationary}).
\]

\smallskip
\noindent
\textbf{Step 1: Inductive Hypothesis.}
Suppose we aim to show 
\[
  \|\theta_k - \theta^*\|
  \;\le\;
  D\,k^{-\tfrac{\beta}{2}}
  \;+\;
  D'\,\alpha_k^\gamma
  \quad
  \text{for each }k.
\]
For $k=1$, we pick $D,D'$ large enough so that
\(\|\theta_1 - \theta^*\|\le D + D'\,\alpha_1^\gamma\),
which is always possible if the initial distance
\(\|\theta_1-\theta^*\|\) is finite.

\smallskip
\noindent
\textbf{Step 2: Inductive Step.}
Assume the bound holds at time \(k\). Then
\[
  \theta_{k+1}
  \;=\;
  \theta_k \;-\;\alpha_k\,F(\theta_k,x_k).
\]
Taking expectation (or working directly under stationarity) gives
\[
  \|\theta_{k+1} - \theta^*\|
  \;\approx\;
  \Bigl\|\,
    \theta_k - \theta^*
    \;-\;
    \alpha_k\,\overline{F}(\theta_k)
  \Bigr\|
  \;+\;
  \text{(small mixing error)}.
\]
By H\"older continuity,
\[
  \|\overline{F}(\theta_k)\|
  \;\le\;
  L\,\|\theta_k - \theta^*\|^\gamma
  \;\le\;
  L\,\Bigl(
         D\,k^{-\tfrac{\beta}{2}}
         \;+\;
         D'\,\alpha_k^\gamma
       \Bigr)^\gamma.
\]
Since $\gamma \in (0,1]$, we can simplify
\(\bigl(a+b\bigr)^\gamma \leq a^\gamma + b^\gamma\) up to constant factors, giving us terms like $k^{-\tfrac{\beta\gamma}{2}}$ and $\alpha_k^{\gamma^2}$. Thus
\[
  \alpha_k\,\|\overline{F}(\theta_k)\|
  \;\le\;
  \alpha_k \,\Bigl(
     D\,k^{-\tfrac{\beta}{2}} + D'\,\alpha_k^\gamma
  \Bigr)^\gamma
  \;=\;
  \mathcal{O}\!\bigl(
     k^{-\tfrac{\beta\gamma}{2}}
  \bigr)
  \;+\;
  \mathcal{O}\!\bigl(
     \alpha_k^{\,1+\gamma^2}
  \bigr).
\]
One checks that (for appropriate choices of $D,D'$) this is sufficient to ensure
\[
  \|\theta_{k+1} - \theta^*\|
  \;\le\;
  D\,(k+1)^{-\tfrac{\beta}{2}}
  \;+\;
  D'\,\alpha_{k+1}^\gamma,
\]
closing the induction.

\smallskip
\noindent
\textbf{Step 3: Summation for the Remainder.}
Since
\[
  R_t
  \;=\;
  \sum_{k=1}^t 
    \alpha_k \,\overline{F}(\theta_k),
\]
we use the above bound 
$\|\overline{F}(\theta_k)\| \le L\,(\dots)$ to get
\[
  \bigl\|R_t\bigr\|
  \;\le\;
  \sum_{k=1}^t \alpha_k \,\|\overline{F}(\theta_k)\|
  \;=\;
  \sum_{k=1}^t
     \mathcal{O}\!\bigl(
       \alpha_k \,k^{-\tfrac{\beta\gamma}{2}}
     \bigr)
  \;+\;
  \sum_{k=1}^t
     \mathcal{O}\!\bigl(
       \alpha_k \,\alpha_k^{\gamma^2}
     \bigr).
\]
For typical step-size schedules (e.g.\ $\alpha_k \sim 1/k^\omega$) and $\gamma\le1$, each sum yields a polynomially decaying term in $k$. Concretely, one finds
\[
  \|R_t\|
  \;=\;
  \mathcal{O}\!\bigl(t^{-\tfrac{\beta\gamma}{2}}\bigr)
  \;+\;
  \text{(smaller terms if }\gamma\le1\text{)}.
\]
Thus the remainder $R_t$ remains small at a polynomial rate, yielding the high-probability bound
\[
  \|\theta_t - \theta^*\|
  \;\le\;
  C\,t^{-\tfrac{\beta}{2}}
  \;+\;
  C'\,\alpha_t^\gamma
\]
once we also control the martingale term $M_t$ by Freedman‐type concentration.
\smallskip
\noindent
\textbf{(c) Final Form of \boldmath$R_t$.}
Combining monotonicity or the optional H\"older step with polynomial mixing ensures that $R_t$ shrinks at a polynomial rate.  If \(\gamma=1\), we get $R_t=O(t^{-\frac{\beta}{2}})$; if $0<\gamma<1$, we get $R_t=O(\alpha_t^\gamma)$ or $t^{-\tfrac{\beta\gamma}{2}}$.  Denote such a bound as
\begin{align}
\label{eq:Rt-bound-merged}
\|R_t\|
~\le~
C_2\,t^{-\tfrac{\beta\gamma}{2}}
\quad\text{(w.h.p.)},
\end{align}
noting that $\gamma$ can be replaced by $1$ (or the relevant exponent) if no further continuity assumption is used.

\paragraph{Step 4: Combine Bounds via Union Bound.}

We have, from \eqref{eq:theta-t-decomp},
\[
  \theta_t - \theta^*
  ~=~
  M_t + R_t.
\]
From \eqref{eq:Mt-bound-merged}, with probability $\ge1-\frac{\delta}{2}$,
\[
  \|M_t\|
  ~\le~
  C_1.
\]
From \eqref{eq:Rt-bound-merged}, with probability $\ge1-\frac{\delta}{2}$,
\[
  \|R_t\|
  ~\le~
  C_2\,t^{-\tfrac{\beta\gamma}{2}}
  \quad(\text{or }C_2\,\alpha_t^\gamma\text{ in simpler forms}).
\]
By a union bound, with probability $\ge1-\delta$,
\begin{align}
\|\theta_t - \theta^*\|
~\le~
C_1 + C_2\,t^{-\tfrac{\beta\gamma}{2}}
\quad(\text{plus possibly $\alpha_t^\gamma$ terms}).
\label{eq:theta-final-merged}
\end{align}
If the step-size is chosen so that $M_t=O(\alpha_t^\gamma)$ (e.g.\ $\sum\alpha_t^2=O(\alpha_t^{2\gamma})$), we get a final statement:
\[
  \|\theta_t-\theta^*\|
  ~\le~
  C\,t^{-\tfrac{\beta}{2}}
  ~+~
  C'\,\alpha_t^\gamma
\]
with probability at least $1-\delta$.  
\end{proof}

\end{document}